\documentclass[preprint,12pt]{elsarticle}

\usepackage{amssymb}
\usepackage{amsthm}

\usepackage[utf8]{inputenc} 
\usepackage[T1]{fontenc}    
\usepackage{hyperref}       
\usepackage{url}            
\usepackage{booktabs}       
\usepackage{amsfonts}       
\usepackage{nicefrac}       
\usepackage{microtype}      
\usepackage{xcolor}         
\usepackage{graphicx}
\usepackage{amsmath}
\usepackage{subfigure}
\usepackage{graphics}
\usepackage{multirow}
\usepackage{mathtools}
\usepackage{tikz-cd}
\usepackage{wrapfig}
\usepackage{adjustbox}
\usepackage{wrapfig}
\usepackage{tablefootnote}

\newtheorem{theorem}{Theorem}

\newtheorem{corollary}[theorem]{Corollary}
\newdefinition{rmk}{Remark}
\newproof{pf}{Proof}
\newproof{pot}{Proof of Theorem \ref{thm}}

\newcommand*{\Scale}[2][4]{\scalebox{#1}{$#2$}}%

\makeatletter
\def\ps@pprintTitle{%
  \let\@oddhead\@empty
  \let\@evenhead\@empty
  \def\@oddfoot{\reset@font\hfil\thepage\hfil}
  \let\@evenfoot\@oddfoot
}
\makeatother

\begin{document}

\begin{frontmatter}



\title{Unveiling the Unseen Potential of Graph Learning through MLPs: Effective Graph Learners Using Propagation-Embracing MLPs}


\author[1]{Yong-Min Shin}

\affiliation[1]{organization={School of Mathematics and Computing (Computational Science and Engineering)},
            addressline={Yonsei Univserity}, 
            city={Seoul},
            postcode={03722}, 
            country={Republic of Korea}}

\author[1,2]{Won-Yong Shin\corref{cor1}}

\affiliation[2]{organization={Graduate School of Artificial Intelligence},
            addressline={Pohang University of Science and Technology (POSTECH)}, 
            city={Pohang},
            postcode={37673}, 
            country={Republic of Korea}}

\cortext[cor1]{Corresponding author.}

\begin{abstract}
Recent studies attempted to utilize multilayer perceptrons (MLPs) to solve semi-supervised node classification on graphs, by training a student MLP by knowledge distillation (KD) from a teacher graph neural network (GNN). While previous studies have focused mostly on training the student MLP by matching the output probability distributions between the teacher and student models during KD, it has not been systematically studied how to inject the structural information in an \textit{explicit} and \textit{interpretable} manner. Inspired by GNNs that separate feature transformation $T$ and propagation $\Pi$, we re-frame the KD process as enabling the student MLP to explicitly learn both $T$ and $\Pi$. Although this can be achieved by applying the inverse propagation $\Pi^{-1}$ before distillation from the teacher GNN, it still comes with a high computational cost from large matrix multiplications during training. To solve this problem, we propose {\bf Propagate \& Distill (P\&D)}, which propagates the output of the teacher GNN before KD and can be interpreted as an approximate process of the inverse propagation $\Pi^{-1}$. Through comprehensive evaluations using real-world benchmark datasets, we demonstrate the effectiveness of {\bf P\&D} by showing further performance boost of the student MLP.
\end{abstract}


\end{frontmatter}

\newpage
\section{Introduction}
\label{section:introduction}

Graph neural networks (GNNs)~\citep{kipf2017gcn, hamilton2017graphsage, velickovic2018gat, xu2019gin, gilmer2017mpnn, bronstein2021messagepassing} have been widely studied as a powerful means to extract useful low-dimensional features from attributed graphs while performing various downstream graph learning tasks such as node classification, link prediction, and community detection. Although aggregating information from neighboring nodes via message passing is crucial to the GNN's performance, this process is also known to cause an exponential increase of inference time with respect to the number of GNN layers~\citep{yan2020tinygnn, zhang2022glnn}. This puts a constraint on the usage of GNNs in various real-world applications especially where fast inference time is essential, such as web recommendation services~\citep{hao2020pcompanion, zhang2020agl}, real-time simulation~\citep{sanchisalepuz2022realtimesimulation}, or even image-guided neurosurgery~\citep{salehi2021physgnn}. Very recently, GLNN~\citep{zhang2022glnn} proposed to use a multilayer perceptron (MLP) as the main architecture for semi-supervised node classification (SSNC) on graphs, by training the student MLP with knowledge distillation (KD)~\citep{hinton2015kd} from a teacher GNN. This approach drastically reduced the inference time, which is more than $\times$100 faster than that of GNNs, while still maintaining satisfactory performance.

Nevertheless, from a KD point of view on the graph domain, choosing MLPs as a student model makes the training task more challenging when compared to other domains such as computer vision and natural language processing. Typically, the student model has access to the same input data as that of the teacher model in the standard KD framework, while benefiting from new information learned the teacher model thanks to its larger capacity~\citep{gou2021kdsurvey}. However, when the student model is denied access to vital parts of the given dataset, {\it i.e.}, structural information on graphs (see Figure~\ref{figure:overview}), there is an even larger information gap between the teacher and student models. Therefore, the main objective of the so-called GNN-to-MLP KD is to enable the weights of the student MLP to \textit{learn} the graph structure so that, during inference, the student MLP achieves the prediction accuracy on par with its counterpart ({\it i.e.}, the teacher GNN).

Although there have been several follow-up studies on GNN-to-MLP KD due to the success of GLNN, common approaches for achieving state-of-the-art performance are based on leveraging structural information as a part of the input to the student MLP~\citep{tian2023nosmog, chen2021samlp, zheng2022coldbrew}, which however poses the following technical challenges:

\begin{itemize}
\item {\bf (Challenge 1)} In~\citep{tian2023nosmog}, which is one follow-up study of GLNN, the positional embeddings learned by DeepWalk~\cite{perozzi2014deepwalk} are concatenated into the node features as input to the student MLP. However, the transductive nature of such embedding techniques may require re-computation when the underlying graph evolves, which thus considerably increases the computational complexity. This makes the distillation from GNNs to MLPs difficult to apply in practice.

\item {\bf (Challenge 2)} Other GNN-to-MLP KD approaches either directly utilize rows of the adjacency matrix, which makes the input dimension of the MLP dependent on the number of nodes~\citep{chen2021samlp}, or require a specific design of GNNs that is able to provide structural node features~\citep{zheng2022coldbrew}. As such, feeding the structural information into the MLP model comes at the cost of extra computation and/or adjustment of dimensionality.
\end{itemize}

\begin{figure*}[t]
\begin{center}
    \includegraphics[width=.95\textwidth]{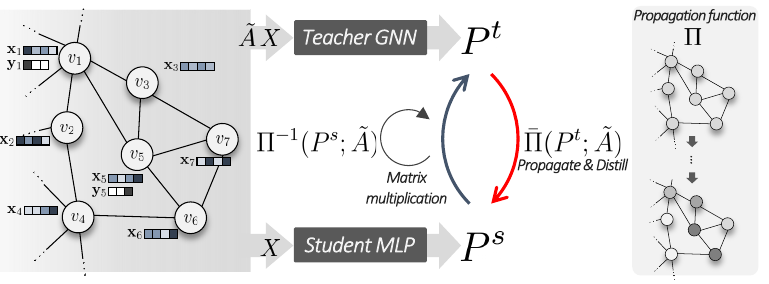}
    \caption{KD from a teacher GNN to a student MLP. In this setting, the figure shows the difference of input information between the teacher GNN and the student MLP, where the student MLP does not have access to the graph structure ({\it i.e.}, the normalized adjacency matrix $\Tilde{A}$) and only uses the node features $X$ as input. In our framework, we propose to enhance the performance of the student model by directly propagating the teacher's output to further inject the graph structure (the blue arrow), which is considerably more computationally efficient than applying an inverse propagation function $\Pi^{-1}$ before distillation (the green arrow) as it requires repeated matrix multiplications during training.}
    \label{figure:overview}
\end{center}
\vskip -0.2in
\end{figure*}

To solve the above limitations and challenges, we devise a new methodology built upon the following two design principles: 1) {\it fixing} the input of the student MLP to the {\it node features only} and 2) injecting structural information during {\it training}. Such a methodology provides a more flexibility in designing GNN-to-MLP KD frameworks. Note that, since most of GNN models stack up to only a few layers ({\it e.g.}, 2 GNN layers)~\citep{shchur2018pitfall, li2019deepgcn} in practice due to the oversmoothing problem~\citep{li2018deeper, chen2020measuring}, depending solely on the teacher GNN's output may disable the student MLP to capture high-order connectivity information. This leaves a room for training the student MLP via KD in the sense of making full use of the graph structure. In this study, we would like to tackle this open challenge in order to unveil the unseen potential of MLPs from the perspective of graph learning.

To achieve this goal, we are basically inspired by GNN approaches that separate feature transformation $T$ and propagation $\Pi$~\citep{gasteiger2019ppnpappnp, huang2021cns, bojchevski2020scalegnnappnp, chien2021generalappnp}, which performs a base prediction followed by a propagation process. Based on these studies, we aim to further boost the performance of the student MLP using $\Pi$ during the KD process in an {\it explicit} and {\it interpretable} manner, which encodes a global structural view of the underlying graph. As illustrated in Figure~\ref{figure:overview}, at its core, this approach begins by regarding the teacher GNN's output as a {\it base prediction} rather than the final prediction. This eventually allows us to arrive at a straightforward formulation where the output of the student MLP first passes through an inverse propagation $\Pi^{-1}$ before being matched with the teacher GNN's output during KD. Although this approach can be interpreted as training the student MLP in such a way that it behaves as a graph learner embracing the propagation $\Pi$ by explicitly learning both $T$ and $\Pi$, it requires large matrix multiplications for each feed-forward process during training in KD. As a more computationally efficient workaround, we propose {\bf Propagate \& Distill (P\&D)}, which approximates $\Pi^{-1}$ by recursively propagating the teacher GNN's output over the graph. Our approach also allows a room for more flexibility by adopting different propagation rules, akin to prior studies on label propagation~\citep{zhou2003lp, zhu2003lp2, huang2021cns}. Through extensive experimental evaluations, we demonstrate the superiority of {\bf P\&D} over benchmark methods on popular real-world benchmark datasets, and show that stronger propagation generally leads to better performance.

In summary, our contributions are as follows:
\begin{itemize}
    \item We present {\bf P\&D}, a simple yet effective GNN-to-MLP distillation method that allows additional structural information to be injected explicitly and interpretably during training by recursively propagating the output of the teacher GNN;
    \item We empirically validate the effectiveness of {\bf P\&D} using real-world graph benchmark datasets for both transductive and inductive settings;
    \item We also provide a case study using a synthetic Chains dataset to interpret the effect of {\bf P\&D} on the distillation performance, while making connections to graph signal denoising (GSD); 
    \item We demonstrate that deeper and stronger propagation generally tends to achieve better performance in {\bf P\&D}, and also provide theoretical findings on the connection between the homophily principle and self-correction;
    \item We compare several variants of the inverse propagation, including a direct convolution, and observe the effect on the student MLP's performance.
\end{itemize}

The remainder of this paper is organized as follows. In Section~\ref{section:relatedwork}, we present prior studies that are related to decoupled GNNs and GNN-to-MLP KD frameworks. In Section~\ref{section:background}, we describe several preliminaries to our work. In Section~\ref{section:proposedmethod}, we explain the methodology of our study, including the background and the proposed {\bf P\&D} framework. In Sections~\ref{section:maineresults} and~\ref{section:analysis}, we present comprehensive experimental results. In Section~\ref{section:theoreticalanalysis}, we also provide a theoretical analysis of our {\bf P\&D} framework. Finally, we provide a summary and concluding remarks in Section~\ref{section:conclusion}.

Table 1 summarizes the notation that is used in this
paper. This notation will be formally defined in the following sections when we introduce preliminaries and our methodology with technical details.

\begin{table}[t]
\caption{Summary of notations.}
\label{table:notation}
\vskip 0.15in
\begin{center}
\begin{small}
\begin{tabular}{ll}
\toprule
\textbf{Notation} & \textbf{Description} \\
\midrule
$G$ & Given graph dataset \\
$\mathcal{V}$ & Set of nodes \\
$\mathcal{V}_{T}$ & Set of nodes with their class labels known during training \\
$\mathcal{U}$ & $\mathcal{V} \setminus \mathcal{V}_T$ \\
$\mathcal{E}$ & Set of edges \\
$\mathcal{Y}$ & Set of class labels \\
$X$ & Node feature matrix \\
$A$ & Adjacency matrix of G \\
$D$ & Degree matrix of $A$ \\
$\Tilde{A}$ & Symmetrically normalized adjacency matrix \\
\bottomrule
\end{tabular}
\end{small}
\end{center}
\vskip -0.1in
\end{table}

\section{Related Work}
\label{section:relatedwork}
In this section, we review previous studies that are most related to our work, including 1) GNN architectures having separate propagation and transformation phases and 2) GNN-to-MLP KD frameworks.

\subsection{GNNs with Decoupled Propagation and Transformation Phases}
\label{subsection:APPNPstyleGNNs}

As a category of neural network architectures that are tailored to process graph data, GNNs are typically characterized by the message passing mechanism~\citep{gilmer2017mpnn}, which defines each GNN layer as a combination of the following two phases: the propagation phase, which utilizes the underlying graph structure to mix the information from each node, and the transformation phase, which transforms the mixed information via a small MLP. By stacking these layers, propagation and transformation phases are alternately performed multiple times throughout the feed-forward process. Aside from these designs, there were a handful of studies that took a different approach by separating propagation and transformation phases throughout the GNN model. As one of the most widely-known such architectures, APPNP~\citep{gasteiger2019ppnpappnp} proposed a design that first transforms the input node features via a simple MLP, followed by several propagation steps akin from PageRank~\citep{Page1999PageRank}; this allows to encode the information from a number of neighboring nodes while avoiding the oversmoothing problem, a phenomenon in message passing neural networks when staking multiple layers. In PPRGo~\citep{bojchevski2020scalegnnappnp}, the authors improved the scalability of APPNP by using a sparse version of the PageRank matrix while enabling parallel and distributed training. GPR-GNN~\citep{chien2021generalappnp} extended APPNP and employed intermediate stages of the PageRank propagation by using their weighted sum with learnable coefficients as the final output rather than just the last propagation step, which enables the model to adapt to graphs with different characteristics (\textit{e.g.}, heterophily). Finally, in Correct and Smooth~\citep{huang2021cns}, the authors proposed to consider residuals to correct the error during propagation, while showing solid performance enhancements across different real-world benchmark datasets. Although our work draws inspiration from these classes of GNNs, we do not directly include propagation as our final model. Rather, our objective is to enable the student MLP model to explicitly learn such propagations so that the propagation functions are a core component only during distillation, but not during inference.

\subsection{GNN-to-MLP KD}
\label{subsection:GNNtoMLPKD}

When GNN-based models are deployed in various real-world applications with strong runtime constraints, one of the practical challenges is the slow inference time caused by the exponential increase in neighborhood fetching required for propagation~\citep{yan2020tinygnn},~\citep{zhang2022glnn}. Recently, GLNN~\citep{zhang2022glnn} proposed a simple but effective solution to this problem, where it combined the fast inference time of MLP models with the empirical capabilities of GNNs by performing KD~\citep{hinton2015kd} during training. The authors demonstrated that distilling the knowledge acquired from the teacher GNN to a student MLP effectively can make MLP models an effective graph learner, despite MLPs not having any explicit access to graph structural information. The core component of effective distillation in this GNN-to-MLP KD framework is how to transfer this structural information to the MLP model. NOSMOG~\citep{tian2023nosmog} employed multiple techniques to address this, which include 1) utilizing graph embedding vectors as input, 2) distilling intermediate representations of the teacher GNN, and 3) adversarial training. SA-MLP~\citep{chen2021samlp} even directly used the adjacency matrix as input. Alternatively, the idea of designing a GNN model that can provide embedding vectors encoded with structural information to be used by the MLP model was presented in Cold Brew~\citep{zheng2022coldbrew}. Recently, it was shown in~\citep{Guo2023LLP} that MLPs can be extended from node classification to link prediction by focusing on relational information between node pairs during KD. As highlighted before, our study is also interested in achieving the goal of making the MLP explicitly learn the structural information. Note that we take a different approach from the aforementioned studies in the sense of avoiding adding additional input to the MLP, which inevitably requires further preprocessing and may hinder computational efficiency during inference.

\section{Preliminaries}
\label{section:background}
In this section, we summarize several preliminaries to our work, along with basic notations.

\subsection{SSNC} 
\label{subsection:SSNC}
In SSNC, we are given a graph dataset $G = (\mathcal{V}, \mathcal{E}, X)$, where $\mathcal{V}$ is the set of nodes, $\mathcal{E} \subseteq \mathcal{V} \times \mathcal{V}$ is the set of edges, and $X \in \mathbb{R}^{|\mathcal{V}| \times d}$ is the node feature matrix where the $i$-th row ${\bf x}_i = X[i,:] \in \mathbb{R}^d$ is the $d$-dimensional feature vector of node $v_i \in \mathcal{V}$. We also denote the adjacency matrix to represent $\mathcal{E}$ as $A \in \mathbb{R}^{|\mathcal{V}| \times |\mathcal{V}|}$, where $A[i,j] = 1$ if $(i,j)\in\mathcal{E}$, and 0 elsewhere. The degree matrix $D = \text{diag}(A{\bf 1}_{|\mathcal{V}|})$ is a diagonal matrix whose diagonal entry $D[i, i]$ represents the number of neighbors for $v_i$, where ${\bf 1}_{|\mathcal{V}|}$ is the all-ones vector of dimension $|\mathcal{V}|$. Alongside $G$, $\mathcal{Y}$ indicates the set of class labels, and each node $v_i$ is associated with a ground truth label $y_i \in \mathcal{Y}$. Typically, $y_i$ is encoded as a one-hot vector ${\bf y}_i \in \mathbb{R}^{|\mathcal{Y}|}$. In SSNC, we assume that only a small subset of nodes $\mathcal{V}_T \subset \mathcal{V}$ have their class labels known during training. The objective of SSNC is to predict the class label for the rest of the nodes in $\mathcal{U} = \mathcal{V} \setminus \mathcal{V}_T$. 

\subsection{KD from GNNs to MLPs} 
Recently, several studies have put their efforts to leverage MLP models as the main architecture for SSNC~\citep{zhang2022glnn, tian2023nosmog, hu2021graphmlp, zheng2022coldbrew, chen2021samlp, Dong2022mlpgraphlearner, wu2023pgkd, Wu2023mlpgraphlearner}. Most of these attempts adopt the KD framework~\citep{bucila2006modelcompression, hinton2015kd, gou2021kdsurvey} by transferring knowledge from a teacher GNN to a student MLP. As the core component, knowledge transfer is carried out by matching soft labels via a loss function $\mathcal{L}_{\text{KL}}$, which plays a role of matching the output probability distributions between the teacher and student models with respect to the Kullback–Leibler (KL) divergence. Although other distillation designs have been proposed since~\citep{hinton2015kd}, distillation via $\mathcal{L}_{\text{KL}}$ has been a popular choice and is adopted in lots of follow-up studies that aim to distill knowledge from GNNs to MLPs.

\begin{figure*}[t]
    \centering
    \subfigure[Overview of the previous GNN-to-MLP KD framework (GLNN).]{
    \centering
    \includegraphics[width=0.75\textwidth]{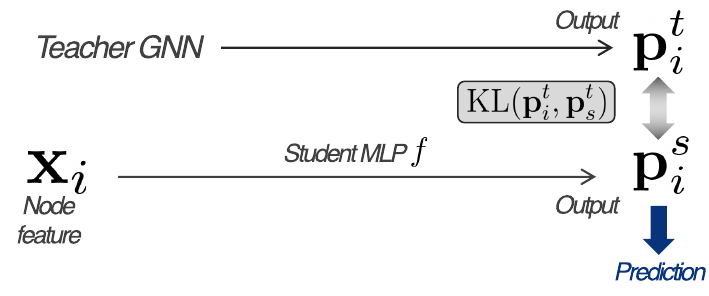}
    \label{figure:main-a}
    }
    \centering
    \subfigure[Our proposed training framework.]{
    \centering
    \includegraphics[width=0.75\textwidth]{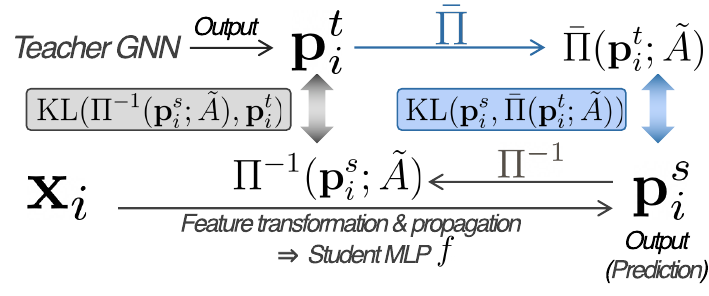}
    \label{figure:main-b}
    }
    \caption{Conceptual comparison between GLNN~\citep{zhang2022glnn} and our proposed \textbf{P\&D}. In Figure~\ref{figure:main-a}, the output probability vector ${\bf p}^t_i$ provided by the teacher GNN is directly used in the KL divergence loss while training the student MLP $f$. On the other hand, Figure~\ref{figure:main-b} illustrates InvKD (gray arrow) and {\bf P\&D} (blue arrow), where the teacher GNN's output ${\bf p}^t_i$ is further propagated before distillation to inject further structural information during training.}
    \label{figure:main}
\end{figure*}

More precisely, in our distillation setting where an MLP is trained via distillation from a teacher GNN (denoted as $g$), we first assume that the output of the teacher GNN, $g(A, X) = H^t \in \mathbb{R}^{|\mathcal{V}| \times |\mathcal{Y}|}$, is given, where ${\bf h}^t_i = H^t[i,:]$ represents the output logit for node $v_i$. The objective of KD from GNNs to MLPs is to train the student MLP model $f$, which returns an output logit $f({\bf x}_i) = {\bf h}^s_i$ for a given node feature vector ${\bf x}_i$ of node $v_i$ as input. The two output logits ${\bf h}_i^t$ and ${\bf h}^s_i$ are transformed into class probability distributions by the softmax function, \textit{i.e.}, ${\bf p}^t_i = \texttt{softmax}({\bf h}^t_i)$ and ${\bf p}^s_i = \texttt{softmax}({\bf h}^s_i)$, respectively. As depicted in Figure~\ref{figure:main-a}, during distillation, $\mathcal{L}_{\text{KL}} \triangleq \text{KL}({\bf p}^s_i, {\bf p}^t_i)$ compares the student's output probability ${\bf p}_i^s$ and the teacher's output probability ${\bf p}_i^t$ by a KL divergence loss. A mix of $\text{KL}({\bf p}^s_i, {\bf p}^t_i)$ and the cross-entropy loss, denoted as $\text{CrossEntropy}({\bf p}^s_i, {\bf y}_i)$, with labeled nodes is used as our final loss function:
\begin{equation}
\label{eq:distillloss}
    \mathcal{L}_{\text{Distill}} = \alpha \sum_{i \in \mathcal{V}_T} \text{CrossEntropy}({\bf p}^s_i, {\bf y}_i) + (1 - \alpha) \sum_{i \in \mathcal{V}} \text{KL}({\bf p}^s_i, {\bf p}^t_i),
\end{equation}
where $\alpha \in [0,1]$ is a mixing parameter. After training, only the MLP model $f$ is used during inference, which dramatically improve the computational efficiency since the feed-forward process basically involves only matrix multiplication and element-wise operations, without message passing~\citep{zhang2022glnn, tian2023nosmog}. Since a majority of GNN-to-MLP KD methods adopt only the second term as their distillation loss~\citep{zhang2022glnn, tian2023nosmog, chen2021samlp}, we also focus on setting $\alpha=0$ in Eq.~(\ref{eq:distillloss}) in our study.

\section{Methodology}
\label{section:proposedmethod}

In this section, we elaborate on our proposed GNN-to-MLP KD framework, named as {\bf P\&D}. We first describe the background of prior work that attempted to separate feature transformation and propagation, and then we describe our {\bf P\&D} framework.

\subsection{Background} Typical GNN models stack multiple message passing layers, each of which consists of the propagation phase and the transformation phase~\citep{gilmer2017mpnn}. On the other hand, a handful of prior studies including~\citep{gasteiger2019ppnpappnp, huang2021cns, bojchevski2020scalegnnappnp, chien2021generalappnp} proposed to separate feature transformation and propagation in the GNN model. Given a feature vector ${\bf x}_i \in \mathbb{R}^{d}$, a feature transformation $T$ is first applied to calculate the base prediction, and then the GNN model further propagates the base prediction along the underlying graph by a propagation function $\Pi$ to get the final prediction. As an example, PPNP~\citep{gasteiger2019ppnpappnp} employed an MLP model to learn the proper feature transformation $T$ and utilized personalized PageRank (PPR) as the propagation function $\Pi$. For some base prediction matrix $H\in\mathbb{R}^{|\mathcal{V}| \times |\mathcal{Y}|}$, the propagation function for $\Pi = \Pi_{\text{PPR}}$ in PPNP is characterized as the matrix multiplication:
\begin{equation}
\label{eq:originalppr}
\Pi_{\text{PPR}}(H; \tilde{A}) = (1 - \gamma)(I_{|\mathcal{V}|} -  \gamma\tilde{A})^{-1}H,
\end{equation}
where $1 - \gamma \in (0, 1]$ is the restart probability, $\mathcal{V}$ is the set of nodes, $I_{|\mathcal{V}|} \in \mathbb{R}^{|\mathcal{V}| \times |\mathcal{V}|}$ is an identity matrix, and $\tilde{A} \in \mathbb{R}^{|\mathcal{V}| \times |\mathcal{V}|}$ is the symmetrically normalized adjacency matrix. Such separation-based approaches have been shown to effectively encode (global) structural information by increasing the number of iterations of propagation while avoiding the oversmoothing effect, a performance deterioration phenomenon known in typical message passing neural network models with deeper layers~\citep{gasteiger2019ppnpappnp, chien2021generalappnp}. 

In our study, we aim to further boost the performance of the student MLP $f$ as a means of making $f$ learn the graph structural information more \textit{explicitly}. More precisely, we are interested in training the student MLP in such a way that it becomes a graph learner embracing the propagation function $\Pi$ by learning both $T$ and $\Pi$. To this end, we re-frame the GNN-to-MLP KD problem by regarding the soft labels provided from a teacher GNN ({\it i.e.}, $P^t$) as a \textit{base prediction} rather than the final prediction for distillation. In this new attempt, we do not want to rely only on the output logits and/or the hidden representations of the teacher model although we are given $P^t$ to aid the training process. Assuming that the feature transformation $T$ on the node features is learned by the student MLP during KD without difficulty, we focus on learning the appropriate $\Pi$ and explicitly introduce an additional propagation function that is applied to $P^t$. Here, we utilize $\Pi(P^t; \tilde{A})$ to take the teacher GNN's output $P^t$ as input and further propagate $P^t$ according to the adjacency matrix $\tilde{A}$. Although the straightforward way to achieve this in the context of KD is to set the loss function as $\mathcal{L}_{\text{KL}} = \text{KL}(P^s, \Pi(P^t; \tilde{A}))$, due to the fact that calculating $\Pi=\Pi_\text{PPR}$ in Eq.~(\ref{eq:originalppr}) without the expensive inverse matrix operation is desired, it is more practical to consider a composite function of both the inverse propagation function $\Pi^{-1}_{\text{PPR}}$ and the student MLP $f$ as our distillation loss. In other words, the propagation is performed on the student model's output $P^s$, \textit{i.e.,} $\Pi^{-1}(P^s; \tilde{A})$, instead of propagating the teacher model's output $P^t$ alongside $\Pi$. Then, we formulate a new training objective in the sense of minimizing the loss $\mathcal{L}_\text{InvKD}$:
\begin{equation}
    \mathcal{L}_\text{InvKD} = \text{KL}(\Pi^{-1}(P^s; \tilde{A}), P^t) = \text{KL}((2 I_{|\mathcal{V}|} - \gamma \tilde{A})P^s, P^t),
    \label{eq:trainingobjective}
\end{equation}
where $\text{KL}$ denotes the KL divergence, $P^s$ is the student MLP's output, and $\gamma$ comes from $\Pi_{\text{PPR}}$.\footnote{The constant term $(1-\gamma)^{-1}$ in $\Pi_{\text{PPR}}^{-1}$ can be ignored as we normalize both terms in the KL divergence loss.} We name the GNN-to-MLP KD framework using Eq.~(\ref{eq:trainingobjective}) as InvKD. In Eq.~(\ref{eq:trainingobjective}), we add an additional identity matrix to $I_{|\mathcal{V}|} - \gamma \tilde{A}$, which can be interpreted as an additional skip-connection alongside $\Pi^{-1}$.\footnote{Our empirical finding showed the substantial performance gain over the case without the additional identity matrix in Eq. (\ref{eq:trainingobjective}).} By multiplying the term $2 I_{|\mathcal{V}|} - \gamma \tilde{A}$ by $P^s$ before calculating the loss, we are capable of making this formulation explicitly involve the structural information during training. The schematic overview of InvKD is illustrated in Figure~\ref{figure:main-b}, where the output prediction of the student MLP $\mathbf{p}_i^s$ is further transformed into $\Pi^{-1}(\mathbf{p}_i^s; \Tilde{A})$, eventually compared with the teacher GNN's output $\mathbf{p}_i^t$ by the KL divergence (see the grey arrow).

\subsection{P\&D} 
In PPNP~\citep{gasteiger2019ppnpappnp}, the authors proposed to approximate the inverse matrix calculation in $\Pi_{\text{PPR}}$ along with a recursive formula. Similarly, instead of using Eq.~(\ref{eq:trainingobjective}) as the distillation loss, we present an alternative of InvKD, named as {\bf P\&D}, due to the computational efficiency and the design flexibility. More specifically, instead of applying the inverse propagation function $\Pi^{-1}$ to the output of the student MLP, we aim to discover an approximate propagation function $\Bar{\Pi} \approx \Pi$, where $\Bar{\Pi}$ is defined as a recursive formula that is applied to the output of the \textit{teacher GNN's prediction}. In other words, $\Bar{\Pi}$ propagates $P^t$ along the underlying graph by recursively applying
\begin{equation}
\label{equation:labelpropagation}
    P^{t}_{l+1} = \gamma \tilde{A} P^t_{l} + (1 - \gamma) P^t_{l},
\end{equation}
where we initially set $P^t_1 = P^t$ for $l = 1, \cdots, T$; and $\gamma \in (0,1]$ is a coefficient controlling the propagation strength through neighbors of each node. Denoting the propagation function in {\bf P\&D} as $\Bar{\Pi}(P^t; \tilde{A})$, we now formulate our new loss function $\mathcal{L}_{\textbf{P\&D}}$ as follows:
\begin{equation}
\label{equation:pndloss}
    \mathcal{L}_{\textbf{P\&D}} = \text{KL}(P^s, \Bar{\Pi}(P^t; \tilde{A})).
\end{equation}
In Figure~\ref{figure:main-b}, the teacher GNN's output $\mathbf{p}_i^t$ is further propagated by $\bar{\Pi}$ before being distilled into the student model (see the blue arrow in Figure~\ref{figure:main-b}). This approach not only introduces another natural \textit{interpretation}, but also allows a room for \textit{flexibility} in the design of a recursive formula. Precisely, Eq.~(\ref{equation:labelpropagation}) can be seen as iteratively smoothing the output of the teacher's prediction along the graph structure, which is closely related to classic label propagation methods~\citep{zhou2003lp, zhu2003lp2} that propagate node label information rather than probability vectors. As in the label propagation, we can say that {\bf P\&D} also takes advantage of the homophily assumption to potentially correct the predictions of incorrectly-predicted nodes with the aid of their (mostly correctly predicted) neighbors. We will theoretically analyze how the homophily assumption plays an important role in Section~\ref{section:theoreticalanalysis}. Furthermore, thanks to the flexibility of the family of label propagation, we introduce another variant, named as \textbf{P\&D}-fix. In this version, the $l$-th iteration of propagation now becomes
\begin{align}
    &\text{(Step 1) } P^{t}_{l+1} =  \gamma \tilde{A}P^t_{l} +(1 - \gamma) P^t_{l},\nonumber\\
    &\text{(Step 2) } P^{t}_{l+1}[j,:] = P^t[j, :] \text{ for } j \in \mathcal{V}_T, \label{equation:fixprop}
\end{align}
where $P^t=P_1^t$ and $\mathcal{V}_T \in \mathcal{V}$ is a subset of nodes having their class labels. The difference between {\bf P\&D} and {\bf P\&D}-fix is that, for every iteration, the output probability of training nodes gets manually replaced by the initial output probability $P^t$ (see Step 2 in Eq.~(\ref{equation:fixprop})). Adding Step 2 during propagation will lead to the initial output probability for some nodes in the training set as their predictions are expected to be nearly correct. In later descriptions, we denote {\bf P\&D} and {\bf P\&D}-fix as the versions using functions $\Bar{\Pi}$ and $\Bar{\Pi}_{\text{fix}}$, respectively. 


Along with \textbf{P\&D} and its variant {\bf P\&D}-fix, we present the following two claims, which generally hold when we evaluate the performance of the student MLP model:

\begin{itemize}
    \item ({\bf C1}) Deeper propagation ({\it i.e.}, higher values of $T$) tends to be beneficial in improving the performance of the student MLP;
    \item ({\bf C2}) Stronger propagation ({\it i.e.}, higher values of $\gamma$) tends to be beneficial in improving the performance of the student MLP.
\end{itemize}

Note that these claims are important in that they provide important guidance on how to conduct propagation before distillation. We shall empirically validate these claims in Section~\ref{subsection:hyperparameterstudy}. 

\section{Main Results}
\label{section:maineresults}

In this section, we present comprehensive experimental results to validate the effectiveness of three GNN-to-MLP KD frameworks, including InvKD, {\bf P\&D}, and {\bf P\&D}-fix.

\subsection{Experimental Setup} 
In our experiments, we mostly follow the settings of prior studies~\citep{zhang2022glnn, yang2021cpf}. Specifically, we set $\alpha = 0$ in Eq.~(\ref{eq:distillloss}) for all experiments, focusing only on the KL divergence loss. We adopt a 2-layer GraphSAGE model~\citep{hamilton2017graphsage} with 128 hidden dimensions. We use the Adam optimizer~\citep{kingma2014adam}, batch size of 512, and early stopping with patience 50 in training. We report the average accuracy of the student MLP model over 10 different trials.

\subsection{Datasets} 
We use the Cora, CiteSeer, Pubmed~\citep{sen2008ccp, yang2016ccp}, A-Computer, and A-Photo~\cite{shchur2018pitfall} datasets. We choose 20 / 30 nodes per class for the training / validation sets as in~\citep{shchur2018pitfall, zhang2022glnn}. For the inductive setting, we further sample 20\% of test nodes to be held out during training. For a larger-scale experimental setting, we also adopt the Arxiv~\citep{hu2020ogb} dataset from the OGB benchmark and use the standard splits. The statistics of the six real-world datasets used in the experiments are summarized in Table~\ref{table:datasetstats}.

\begin{table}[t]
\caption{Statistics of six real-world datasets. NN, NE, NF, and NC denote the number of nodes, the number of edges, the number of node features, the number of classes, respectively.}
\label{table:datasetstats}
\vskip 0.15in
\begin{center}
\begin{small}
\begin{tabular}{lcccc}
\toprule
Dataset & NN & NE  & NF & NC \\
\midrule
Cora & 2,485 & 5,069 & 1,433 & 7\\
CiteSeer & 2,120 & 3,679 & 3,703 & 6\\
Pubmed & 19,717 & 44,324 & 500 & 3\\
A-computer & 13,381 & 245,778 & 767 & 10\\
A-photo & 7,487 & 119,043 & 745 & 8\\
Arxiv & 169,343 &  1,166,243 & 128 & 40\\
\bottomrule
\end{tabular}
\end{small}
\end{center}
\vskip -0.1in
\end{table}

\subsection{Scenario Settings} 
\label{subsection:scenariosettings}
In our study, we consider transductive and production scenarios. In the transductive scenario, we assume that access to information besides the labels (\textit{i.e.}, node features and associated edges) for all nodes in $\mathcal{U}$ is available during training. In other words, we use the set of edges $\mathcal{E}$ and node features $X$ for all nodes in $\mathcal{V}$, along with the label information in $\mathcal{V}_T\subset \mathcal{V}$ to be used during training.

In the production scenario, we sample a held-out subset of nodes in $\mathcal{U}$ by separating into two disjoint subsets, namely the observed subset $\mathcal{U}_{\text{obs}}$ and the inductive subset $\mathcal{U}_{\text{ind}}$ \textit{i.e.,} $\mathcal{U}_{\text{obs}} \cup \mathcal{U}_{\text{ind}} = \mathcal{U}$ and $\mathcal{U}_{\text{obs}} \cap \mathcal{U}_{\text{ind}} = \varnothing$. All edges that connect nodes between $\mathcal{U}_{\text{obs}}$ and $\mathcal{U}_{\text{ind}}$ is removed, and remain disconnected during the test phase, following~\citep{zhang2022glnn}. The unseen nodes in $\mathcal{U}_{\text{ind}}$ are completely unknown during training. The experiments for the production scenario are carried out in Section~\ref{subsection:productionexperiment}.


\subsection{Experimental Results and Analyses} 
Table~\ref{table:mainexperiment} summarizes the performance comparison of InvKD, {\bf P\&D}, and {\bf P\&D}-fix with three benchmark methods, including GLNN~\citep{zhang2022glnn}, the teacher GNN model, and the plain MLP model without distillation, in terms of the node classification accuracy for five real-world graph benchmark datasets, including Cora, CiteSeer, Pumbed, A-Computer, and A-Photo, in the transductive setting. Although there are several follow-up studies on GNN-to-MLP KD, only GLNN adopts the fair setting as ours. Note that a more recent benchmark method such as NOSMOG~\citep{tian2023nosmog} was not shown since it utilizes graph embedding techniques to provide additional input to the student MLP, which is rather unfair in comparison with our approach. In the experimental result, we observe that using one of InvKD, \textbf{P\&D}, and \textbf{P\&D}-fix consistently outperforms all the benchmark methods regardless of datasets. For example, when the CiteSeer dataset is used, {\bf P\&D}-fix exhibits the best performance with the gain of 3.74\% over GLNN.

\begin{table*}[t]
\caption{Node classification accuracy (\%) for five different datasets in the transductive setting. The columns represent the performance of the teacher GNN model, plain MLP model without distillation, GLNN~\citep{zhang2022glnn}, InvKD, and two versions of {\bf P\&D}. For each dataset, the performance of the best method is denoted in bold font.}
\vskip 0.1in
\label{table:mainexperiment}
\centering
\resizebox{1\textwidth}{!}{%
\begin{tabular}{lcccccc}
\toprule
\textit{Transductive} & Teacher GNN & Plain MLP & GLNN & InvKD & {\bf P\&D} & {\bf P\&D}-fix  \\
\midrule
Cora & 78.81 ± 2.00 & 59.18 ± 1.60 & 80.73 ± 3.42 & 82.22 ± 1.45 & 82.16 ± 1.98 & \textbf{82.29} ± 1.60 \\
CiteSeer & 70.62 ± 2.24 & 58.51 ± 1.88 & 71.19 ± 1.36 &  74.08 ± 1.82 & 73.38 ± 1.39 & \textbf{74.93} ± 1.63 \\
Pubmed & 75.49 ± 2.25 & 68.39 ± 3.09 & 76.39 ± 2.36 & 77.22 ± 1.98 & 77.88 ± 2.89 & \textbf{78.11} ± 2.89 \\
A-Computer & 82.69 ± 1.26 & 67.79 ± 2.16 & 83.61 ± 1.49 & \textbf{83.81} ± 1.16 & 82.06 ± 1.58 & 83.21 ± 1.21 \\
A-Photo & 90.99 ± 1.34 & 77.29 ± 1.79 & 92.72 ± 1.11 &  92.83 ± 1.22 & 92.91 ± 1.31 & \textbf{93.02} ± 1.32 \\ 
\bottomrule
\end{tabular}
}
\end{table*}

\begin{table*}[!t]
\caption{Node classification accuracy (\%) for the Arxiv dataset in the transductive setting. The columns represent the performance of the teacher GNN model, plain MLP model without distillation, GLNN~\citep{zhang2022glnn}, and two versions of {\bf P\&D}. The performance of the best method is denoted in bold font.}
\vskip 0.1in
\label{table:mainexperimentarxiv}
\centering
\resizebox{1\textwidth}{!}{%
\begin{tabular}{lccccc}
\toprule
\textit{Transductive} & Teacher GNN & Plain MLP & GLNN  & {\bf P\&D} & {\bf P\&D}-fix  \\
\midrule
Arxiv & 70.64 ± 0.41 & 55.33 ± 1.54 & 63.02 ± 0.41 & \textbf{65.20} ± 0.45 & 65.14 ± 0.35 \\
\bottomrule
\end{tabular}
}
\end{table*}

Additionally, Table~\ref{table:mainexperimentarxiv} also summarizes the performance comparison of \textbf{P\&D} and \textbf{P\&D}-fix with three benchmark methods, including GLNN, the teacher GNN, and the plain MLP, in terms of the node classification accuracy on the larger Arxiv dataset with more than 1 million edges in the transductive setting.\footnote{Since InvKD requires large matrix multiplications during training, it is computationally expensive and memory-inefficient. Thus, we have not run InvKD for the arXiv dataset.} Table~\ref{table:mainexperimentarxiv} shows that \textbf{P\&D} outperforms all the benchmark methods by revealing the gain of 2.18\% over GLNN.

\begin{table*}[!t]
\caption{Node classification accuracy (\%) for the Cora, CiteSeer, Pubmed, A-Computer, and A-Photo dataset in the production setting. The columns represent the performance of the teacher GNN model, plain MLP model without distillation, GLNN~\citep{zhang2022glnn}, InvKD, and two versions of {\bf P\&D}. The performance of the best method is denoted in bold font.}
\vskip 0.1in
\label{table:mainexperimentnosmog}
\centering
\resizebox{1\textwidth}{!}{%
\begin{tabular}{lllllllll}
\toprule
Datasets & Eval & Teacher GNN & Plain MLP & GLNN  & InvKD & {\bf P\&D} & {\bf P\&D}-fix  \\
\midrule
\multirow{3}{*}{Cora} & \textit{prod} & 79.17 & 59.18 & 77.81 & 79.35 & \textbf{79.40} & 78.89 \\ 
& \textit{ind} & 80.61 ± 1.81 & 59.44 ± 3.36 & 72.55 ± 2.58 & \textbf{75.18} ± 1.26 & 72.27 ± 2.74 & 71.24 ± 3.45 \\
& \textit{tran} & 78.81 ± 0.20 & 59.12 ± 1.49 & 79.13 ± 2.01 & 80.40 ± 2.16 & \textbf{81.18} ± 2.04 & 80.81 ± 2.17 \\ \midrule
\multirow{3}{*}{CiteSeer} & \textit{prod} & 68.60 & 58.51 & 68.83 & 72.81 & \textbf{73.76} & 73.20 \\ 
& \textit{ind} & 69.83 ± 4.16 & 59.34 ± 4.61 & 68.37 ± 4.22 & 71.93 ± 3.16 & \textbf{72.87} ± 2.57 & 72.79 ± 2.91 \\
& \textit{tran} & 68.29 ± 3.20 & 58.30 ± 1.95 & 68.94 ± 3.47 & 73.03 ± 2.98 & \textbf{73.98} ± 2.52 & 73.30 ± 2.53 \\ \midrule
\multirow{3}{*}{Pubmed} & \textit{prod} & 74.99 & 68.39 & 75.15 & 76.26 & 76.95 & \textbf{77.15} \\ 
& \textit{ind} & 75.25 ± 2.42 & 68.29 ± 3.26 & 75.01 ± 2.20 & 75.43 ± 2.36 & 76.49 ± 2.47 & \textbf{76.58} ± 2.34 \\
& \textit{tran} & 74.92 ± 1.91 & 68.42 ± 3.06 & 75.18 ± 2.08 & 76.47 ± 2.18 & 77.07 ± 2.36 & \textbf{77.29} ± 2.26 \\
\midrule
\multirow{3}{*}{A-Computer} & \textit{prod} & 83.04 & 67.79 & 82.56 & 82.54 & \textbf{83.26} & 83.23 \\ 
& \textit{ind} & 83.06 ± 1.81 & 67.86 ± 2.16 & 79.77 ± 1.72 & 80.04 ± 1.90 & 80.28 ± 1.79 & \textbf{80.38} ± 1.59 \\
& \textit{tran} & 83.04 ± 1.60 & 67.77 ± 2.18 & 83.26 ± 1.34 & 83.17 ± 1.84 & \textbf{84.01} ± 1.77 & 83.94 ± 1.56 \\
\midrule
\multirow{3}{*}{A-Photo} & \textit{prod} & 90.93 & 77.29 & 91.56 & \textbf{92.38} & 91.81 & 92.01 \\ 
& \textit{ind} & 91.21 ± 1.10 & 77.44 ± 1.50 & 89.73 ± 1.18 & \textbf{90.28} ± 1.04 & 90.23 ± 1.02 & 89.87 ± 1.03 \\
& \textit{tran} & 90.86 ± 0.71 & 77.25 ± 1.90 & 92.07 ± 0.78 & \textbf{92.91} ± 0.66 & 92.20 ± 0.68 & 92.54 ± 0.08 \\
\bottomrule
\end{tabular}
}
\end{table*}

\subsection{Experimental Results under Production Scenarios}
\label{subsection:productionexperiment}
We run further experiments in order to observe the performance of our proposed framework in the production scenario, which provides a comprehensive view of the performance in both transductive and inductive settings. We measure three performance scores, \textit{i.e.}, transductive, inductive, and production scores. The transductive and inductive scores are the accuracy measured on the observed nodes $\mathcal{U}_{\text{obs}}$ and the unseen nodes $\mathcal{U}_{\text{ind}}$, respectively. To evaluate the performance under the production scenario, we interpolate the performance on the observed nodes and the unseen nodes by 8:2.

Table~\ref{table:mainexperimentnosmog} summarizes the performance comparison of InvKD, \textbf{P\&D}, and \textbf{P\&D}-fix with three benchmark methods, including GLNN~\cite{zhang2022glnn}, the teacher GNN, and the plain MLP, in terms of the node classification accuracy for five benchmark datasets, including Cora, CiteSeer, Pubmed, A-Computer, and A-Photo, in the production setting. It is obvious to see that either InvKD, \textbf{P\&D}, or \textbf{P\&D}-fix is the best performer, consistently outperforming GLNN. In particular, we observe that our framework still brings performance benefits in the inductive scenario; for example, InvKD exhibits $2.63\%$ higher performance than that of GLNN in the inductive setting for the Cora dataset. This implies that the additional structural information also facilitates precise prediction even for unseen nodes.

\section{Further Experimental Analyses}
\label{section:analysis}

In this section, we carry out comprehensive experiments to clearly visualize the effects of inverse/and approximate recursive propagation functions ({\it i.e.}, $\Pi^{-1}$ and $\bar{\Pi}$, respectively) for interpretation. We also make connections with GSD, while empirically analyzing how different propagation settings affect the distillation performance in \textbf{P\&D}. Furthermore, we analyze the effect of $\Pi^{-1}$ in InvKD in comparison with alternative operations. Finally, we investigate performance gains when we employ the teacher GNN model with decoupled propagation and transformations.

\begin{figure}[t]
\centering
  \includegraphics[width=0.9\columnwidth]{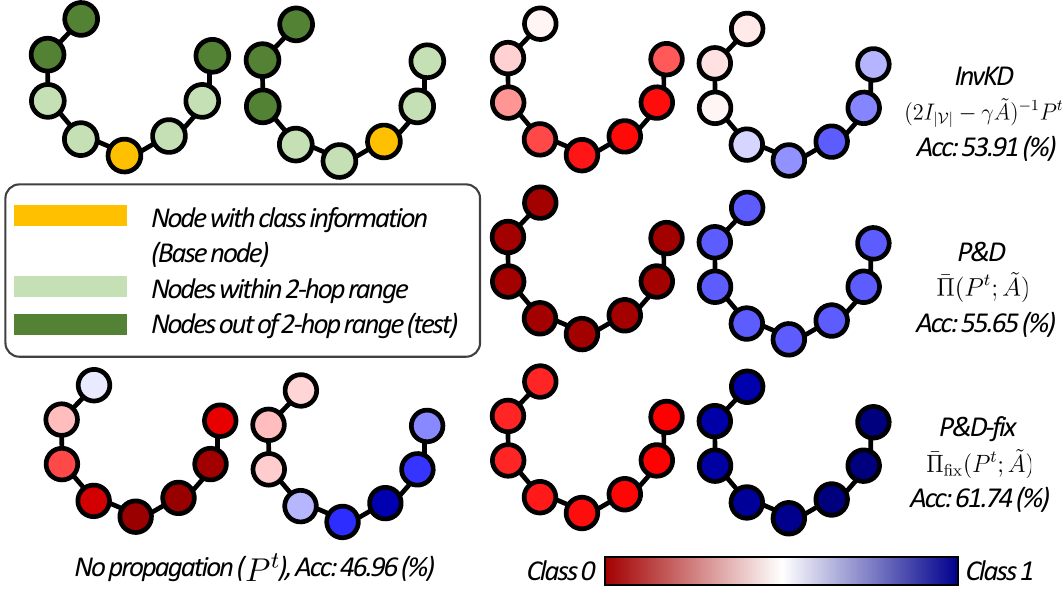}
  \caption{Visualization of the effect of various propagation functions in InvKD, {\bf P\&D}, and {\bf P\&D}-fix on the synthetic Chains dataset.}\label{fig:chainsexperiment}
\end{figure}

\subsection{Case Study on Interpretations}
\label{subsection:casestudy}
In order to provide interpretations on the benefits of injecting additional structural information in InvKD and \textbf{P\&D} during distillation, we perform experiments on the synthetic Chains dataset~\cite{gu2020ignn}, which consists of 30 chain graphs of a fixed length of 8. Fig.~\ref{fig:chainsexperiment} visualizes 2 chains for each propagation for ease of presentation. All nodes in the same chain are assigned to the same class, and the class information is provided as a one-hot representation in the feature vector for \textit{one} of the nodes (the base node) in the chain. To train the teacher GNN, we adopt a 2-layer GraphSAGE model, which is thus able to exploit connectivities only within 2-hop neighbors of the base node. Then, we plot $P^t$, $(2I_{|\mathcal{V}|} -  \gamma\tilde{A})^{-1} P^t$, $\Bar{\Pi}(P^t; \tilde{A})$, and $\Bar{\Pi}_{\text{fix}}(P^t; \tilde{A})$, which correspond to the class probability distributions for GLNN, InvKD, \textbf{P\&D}, \textbf{P\&D}-fix, respectively.\footnote{Although we do not directly use $(2I_{|\mathcal{V}|} -  \gamma\tilde{A})^{-1} P^t$ in the loss for InvKD, we can regard this as the final prediction when the student MLP ideally achieves the zero loss during training.} Compared to $P^t$ where the teacher GNN only correctly predicts the nodes near the base node, other propagation cases $(2I_{|\mathcal{V}|} -  \gamma\tilde{A})^{-1} P^t$, $\bar{\Pi}(P^t; \tilde{A})$, and $\bar{\Pi}_\text{fix}(P^t; \tilde{A})$ further spread the correct label information along the graph, while self-correcting the base prediction of $P^t$. Additionally, we evaluate the accuracy of the student MLP on the nodes further than 2-hops away from the base node (see dark green nodes in the left part of Fig.~\ref{fig:chainsexperiment}). We can observe that self-correction indeed benefits the student MLP. For example, for the nodes out of the 2-hop range, using $\Bar{\Pi}_{\text{fix}}(P^t; \tilde{A})$ ({\it i.e.}, {\bf P\&D}-fix) results in the accuracy of $61.74\%$ compared to the case of using $P^t$ showing the accuracy of $46.96\%$. This case study clearly validates the effect of our inverse and approximate recursive propagation functions ({\it i.e.}, $\Pi^{-1}$ and $\bar{\Pi}$, respectively).


\begin{figure}[t]
    \centering
    \subfigure[Cora.]{
    \centering
    \includegraphics[width=0.3\textwidth]{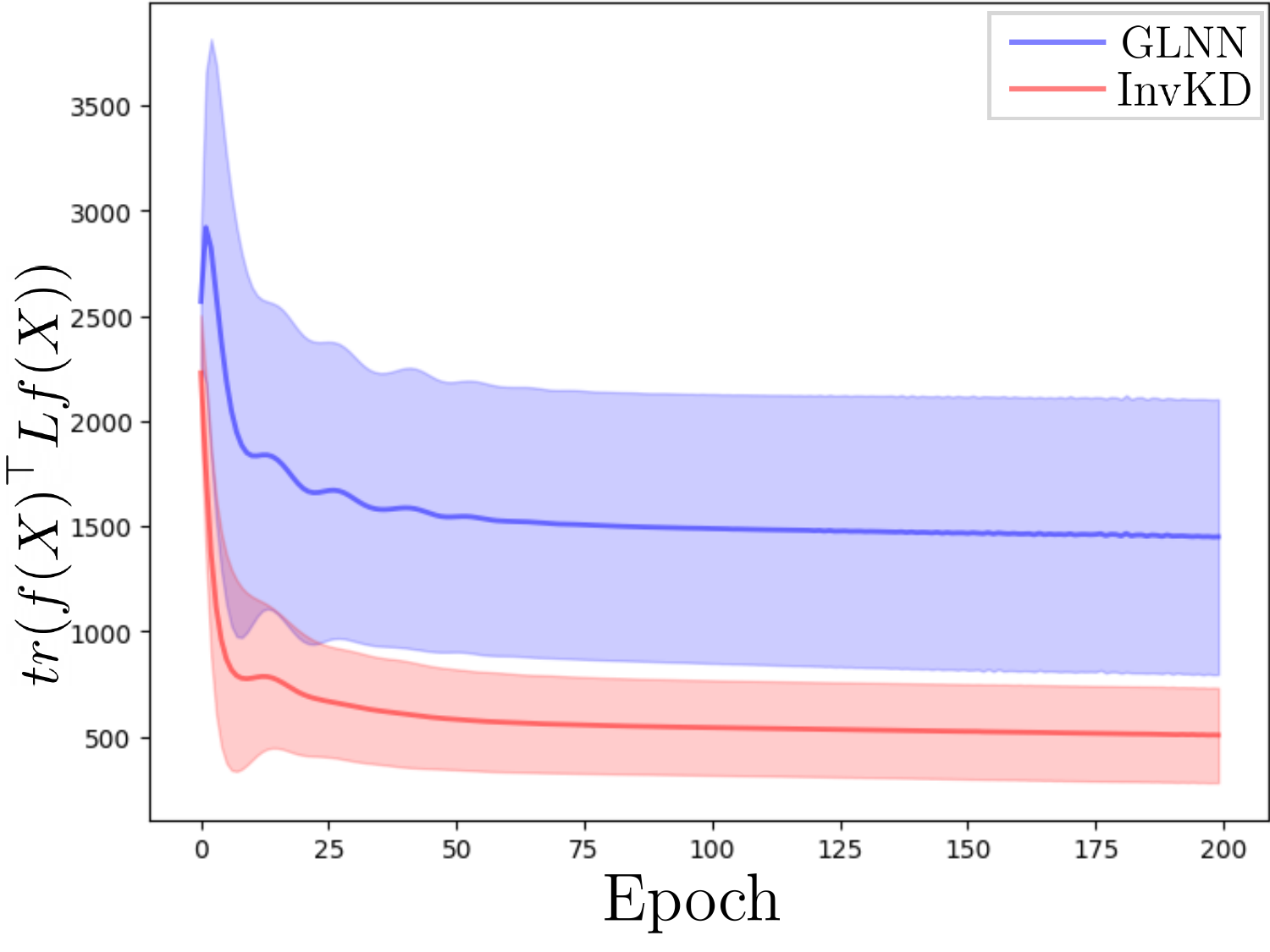}
    \label{figure:app-lap-a}
    }
    \centering
    \subfigure[CiteSeer.]{
    \centering
    \includegraphics[width=0.3\textwidth]{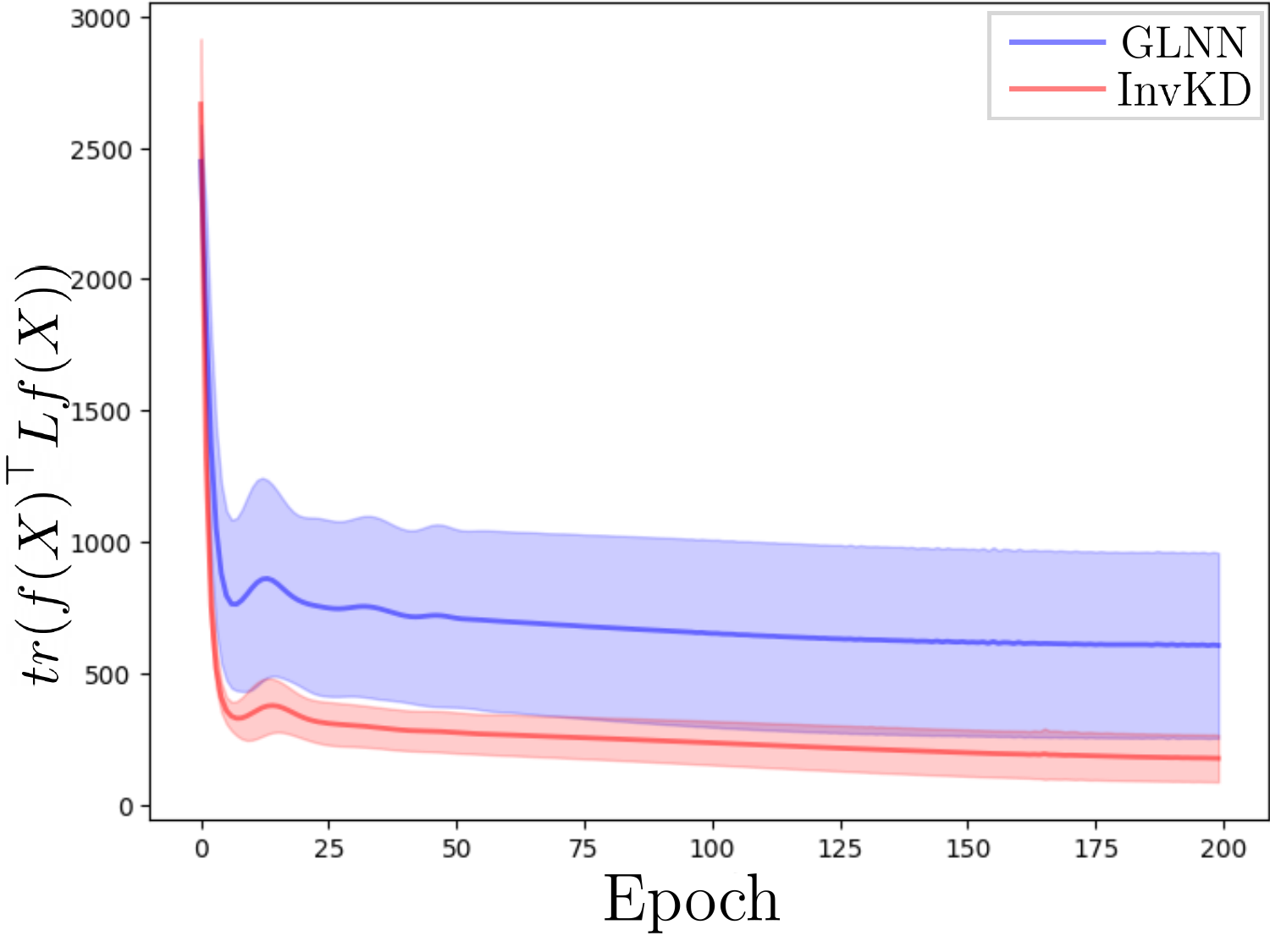}
    \label{figure:app-lap-b}
    }
    \centering
    \subfigure[PubMed.]{
    \centering
    \includegraphics[width=0.3\textwidth]{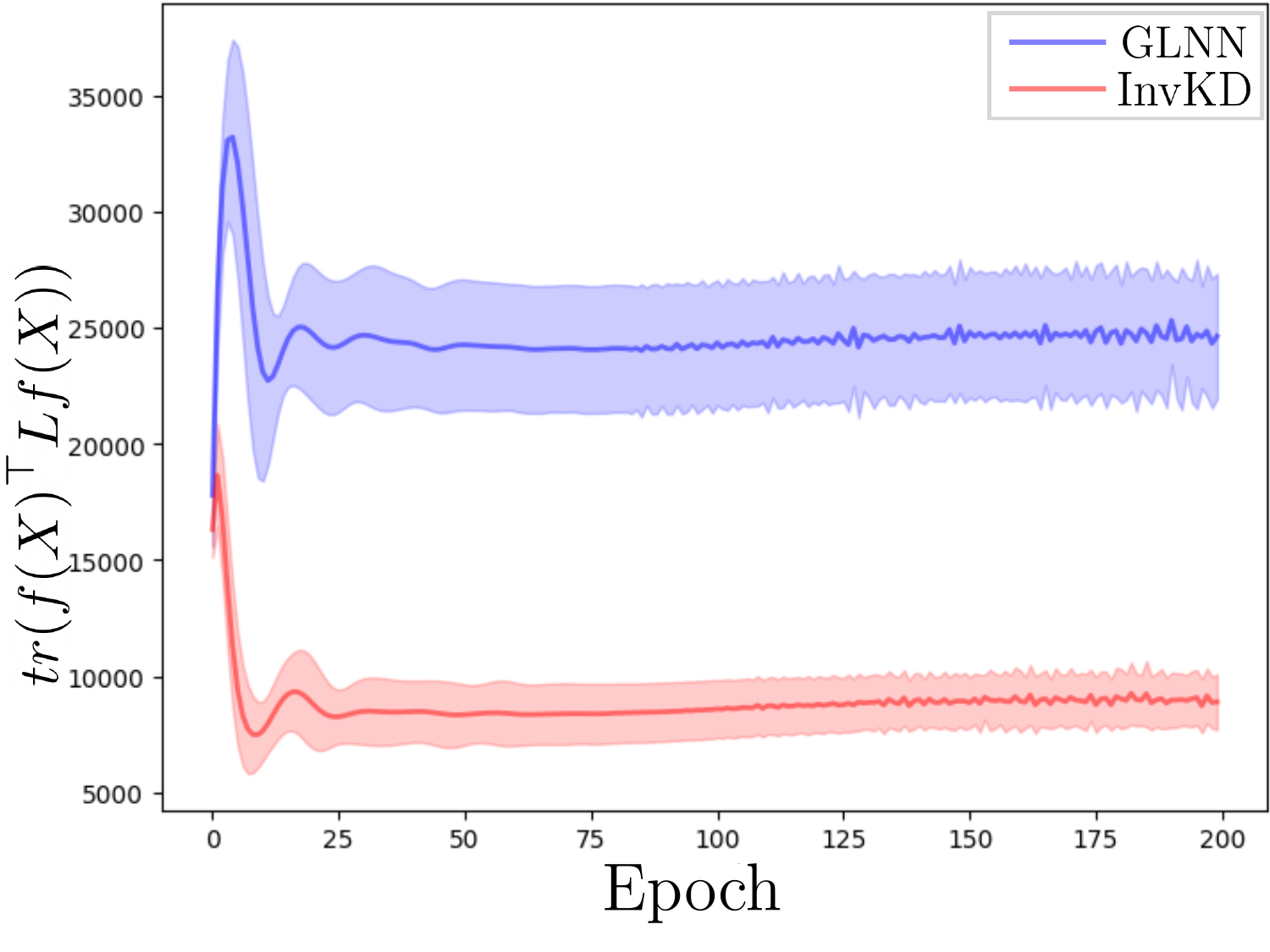}
    \label{figure:app-lap-c}
    }
    \caption{$\text{tr}(f(X)^\top L f(X))$ versus the number of epochs for three different datasets.}
    \label{fig:laplacianvsepoch}
\end{figure}

\subsection{Graph Smoothing Point of View of I\&D} 

We can interpret InvKD as GSD on the student MLP. For such interpretation, we first present the following theorem that connects propagation and GSD:
\begin{theorem}[GSD of PPNP~\citep{ma2021gsd}]
\label{theorem:originalgraphsignaldenoising}
    Given a noisy signal $S \in \mathbb{R}^{|\mathcal{V}| \times |\mathcal{Y}|}$, PPNP~\citep{gasteiger2019ppnpappnp} solves a GSD problem, where the goal is to recover a clean signal $F \in \mathbb{R}^{|\mathcal{V}| \times |\mathcal{Y}|}$ by solving the following optimization problem: 
    \begin{equation}
    \label{ep:graphsignaldenoisingobjective}
        \text{arg}\min_{F} \mathcal{L}_{\text{GSD}} = ||F - S||^2 + (1/(1 -\gamma) - 1) \text{tr}(F^\top L F),
    \end{equation}
    where $\text{tr}(\cdot)$ is the trace of a matrix and $L = I_{|\mathcal{V}|} - \tilde{A}$ is the Laplacian matrix.
\end{theorem}
By interpreting $F = f(X)$ and $S = P^t$ in Theorem~\ref{theorem:originalgraphsignaldenoising}, InvKD can be seen as training the student MLP to fit a given signal $P^t$ with an additional smoothing constraint $(1/(1 -\gamma) - 1) \text{tr}(F^\top L F)$. To explicitly see this effect, we perform an experiment on the Cora, CiteSeer, and PubMed datasets where we plot the regularization term $\text{tr}(f(X)^\top L f(X))$ versus the number of epochs during training for GLNN and InvKD in Fig.~\ref{fig:laplacianvsepoch}. We can observe a much stronger regularization effect for the case of InvKD (red) compared to na\"ive KD (blue). In conclusion, the inverse propagation $\Pi^{-1}$ in InvKD further forces the student MLP to return a signal smoothened over the graph.

\begin{table}[t]
\centering
\caption{Node classification accuracy (\%) according to different $T$'s for three different datasets for {\bf P\&D} and {\bf P\&D}-fix. The best performing cases are underlined, and the performance gain over the case of $T\in\{1,2,5\}$ is displayed in the parenthesis.}
\vskip 0.1in
\label{table:analysisTmini}
\resizebox{0.8\textwidth}{!}{%
\begin{tabular}{lccccccc}
\toprule
 &  & \multicolumn{3}{c}{\textbf{P\&D}} & \multicolumn{3}{c}{\textbf{P\&D}-fix} \\
\cmidrule{3-8}
 $T$ & & Cora & CiteSeer & PubMed & Cora & CiteSeer & PubMed \\
\midrule
\multirow{2}{*}{$\leq$5}& {\em Trans.} & 82.16 & 73.72 & 76.68 & 81.64 & 74.93 & 77.14\\
                        & {\em Ind.} & 71.59 & \underline{72.87} & 76.49 & 71.24 & 72.69 & 76.39\\
\midrule
\multirow{4}{*}{10}     & \multirow{2}{*}{{\em Trans.}} & 82.88 & 73.65 & \underline{77.88} & \underline{82.35} & 74.01 & 77.06 \\
                        & & ($\uparrow$0.72)  & ($\downarrow$0.07) & ($\uparrow$1.20) & ($\uparrow$0.71)  & ($\downarrow$0.92) & ($\downarrow$0.08)\\
\cmidrule{2-8}
                        & \multirow{2}{*}{{\em Ind.}} & \underline{72.27} & 72.21 & 76.57 & 70.59 & 70.59 & \underline{76.59}\\
                        & & ($\uparrow$0.68)  & ($\downarrow$0.66) & ($\uparrow$0.08) & ($\downarrow$0.65)  & ($\downarrow$2.10) & ($\uparrow$0.20)\\
\midrule
\multirow{4}{*}{20}    & \multirow{2}{*}{{\em Trans.}} & \underline{83.03} & \underline{73.74} & 76.56 & 81.85 & 74.04 & 77.54\\
                       & & ($\uparrow$0.87)  & ($\uparrow$0.02) & ($\downarrow$0.12) & ($\uparrow$0.21)  & ($\downarrow$0.89) & ($\uparrow$0.40)\\
\cmidrule{2-8}
                        & \multirow{2}{*}{{\em Ind.}} & 71.31 & 72.21 & \underline{76.62} & \underline{71.85} & 71.85 & 76.36\\
                        & & ($\downarrow$0.28)  & ($\downarrow$0.66) & ($\uparrow$0.13) & ($\uparrow$0.61)  & ($\downarrow$0.84) & ($\downarrow$0.03)\\
\midrule
\multirow{4}{*}{50}     & \multirow{2}{*}{{\em Trans.}} & 82.38 & 73.38 & 77.01 & 82.29 & \underline{74.97} & \underline{78.11}\\
                        & & ($\uparrow$0.22)  & ($\downarrow$0.34) & ($\uparrow$0.33) & ($\uparrow$0.65)  & ($\uparrow$0.04) & ($\uparrow$0.97)\\
\cmidrule{2-8}
                        & \multirow{2}{*}{{\em Ind.}} & 71.85 & 71.77 & 76.48 & 71.66 & \underline{72.76} & 76.58\\
                        & & ($\uparrow$0.26)  & ($\downarrow$1.10) & ($\downarrow$0.01) & ($\uparrow$0.42)  & ($\uparrow$0.07) & ($\uparrow$0.19)\\
\bottomrule
\end{tabular}
}
\end{table}

\subsection{Does Deeper and Stronger Propagation Result in Better MLPs?}
\label{subsection:hyperparameterstudy}
We also investigate how the total number of iterations $T$ and the propagation strength $\gamma$ in Eq.~(\ref{equation:labelpropagation}) affects the performance to validate our claims {\bf (C1)} and {\bf (C2)} in Section~\ref{section:proposedmethod}. Here, we perform the analysis along with {\bf P\&D} as our main framework. To see how the performance behaves with $T$, we consider four cases: $T \in \{1,2,5\}$, $T = 10$, $T = 20$, and $T = 50$.\footnote{For $T \in \{1,2,5\}$, we choose the one leading to the best performance.} We evaluate the performance gain over to the first case ({\it i.e.}, $T \in \{1,2,5\}$) using the Cora, Citeseer, and Pubmed datasets. Table~\ref{table:analysisTmini} shows that, in both transductive/inductive settings, the best performance can be achieved when $T$ is sufficiently large ({\it i.e.}, $T\ge 10$).\footnote{The inductive settings is identical to the one in the production setting.} Next, to see how the performance behaves with $\gamma$, we consider two cases: $\gamma = 0.1$ and $\gamma = 0.9$. Table~\ref{table:analysisGmini} shows that stronger propagation ({\it i.e.}, $\gamma=0.9$) leads to higher performance in both cases.


\begin{table}[t]
\caption{Node classification accuracy (\%) according to different $\gamma$'s for three different datasets for \textbf{P\&D} and \textbf{P\&D}-fix. The performance gain of the case of $\gamma = 0.9$ over the case of $\gamma = 0.1$ is displayed in the parenthesis.}
\vskip 0.1in
\label{table:analysisGmini}
\centering
\small
\resizebox{1\textwidth}{!}{%
\begin{tabular}{lcccccccc}
\toprule
    & \multicolumn{4}{c}{\textbf{P\&D}} & \multicolumn{4}{c}{\textbf{P\&D}-fix} \\
\cmidrule{2-9}
\multirow{2}{*}{Dataset}& \multicolumn{2}{c}{{\em Transductive}} & \multicolumn{2}{c}{{\em Inductive}} & \multicolumn{2}{c}{{\em Transductive}} & \multicolumn{2}{c}{{\em Inductive}}\\
\cmidrule{2-9}
                        & $\gamma = 0.1$ & $\gamma = 0.9$ & $\gamma = 0.1$ & $\gamma = 0.9$ & $\gamma = 0.1$ & $\gamma = 0.9$ & $\gamma = 0.1$ & $\gamma = 0.9$\\
\midrule
\multirow{2}{*}{Cora}    & 80.35 & 82.16 & 70.87 & 71.99 & 80.85 & 82.35 & 69.93 & 71.24\\
                         & & ($\uparrow$1.81) & & ($\uparrow$1.12) & & ($\uparrow$1.50) & & ($\uparrow$1.31) \\
\midrule
\multirow{2}{*}{CiteSeer}    & 72.70 & 73.38 & 71.60 & 72.87 & 74.47 & 74.93 & 69.93 & 72.69\\
                         & & ($\uparrow$0.68) & & ($\uparrow$1.27) & & ($\uparrow$0.46) & & ($\uparrow$2.76) \\
\midrule
\multirow{2}{*}{PubMed}    & 76.56 & 77.88 & 75.84 & 76.49 & 76.89 & 78.11 & 75.79 & 76.58\\
                         & & ($\uparrow$1.32) & & ($\uparrow$0.65) & & ($\uparrow$1.22) & & ($\uparrow$0.79) \\
\bottomrule
\end{tabular}
}
\end{table}


\subsection{Replacing $\Pi^{-1}$ with Alternative Operations}

\begin{table}[t]
\caption{Performance comparison when three different loss functions are used during training in the transductive setting on three different datasets.}
\label{table:indanalysis}
\begin{center}
\begin{small}
\begin{tabular}{lccc}
\toprule
Dataset  & $\mathcal{L}_{\text{conv}}$ & $\mathcal{L}_{\text{Distill}}$ & $\mathcal{L}_{InvKD}$  \\
\midrule
Cora & 64.78 ± 0.92 & 73.84 ± 0.52 & 76.20 ± 0.48 \\
CiteSeer & 69.71 ± 0.48 & 70.23 ± 0.40 & 72.43 ± 0.56 \\
PubMed & 60.27 ± 9.64 & 76.77 ± 1.02  & 77.67 ± 0.73 \\
\bottomrule
\end{tabular}
\end{small}
\end{center}
\vskip -0.1in
\end{table}

In InvKD, our inverse propagation function $\Pi^{-1}$ reveals a form similar to the Laplacian matrix $I_{|\mathcal{V}|}-\tilde{A}$ (see Eq.~(\ref{eq:trainingobjective})). Then, a natural question raising is ``Is $\Pi^{-1}$ in InvKD replaceable with alternative operations during distillation?''. To answer this question, we run a simple experiment by taking into account two alternatives. First, one can expect that the student MLP model $f$ may also learn the structural information when a convolution function ({\it i.e.}, the symmetrically normalized adjacency matrix) is applied to the output of the student MLP instead of $\Pi^{-1}$ since the gradients of the model parameters are also influenced by $\tilde{A}$. Thus, we consider training the student MLP with $\mathcal{L}_{\text{conv}} \triangleq \text{KL}(\tilde{A}P^s, P^t)$. Second, as another alternative, we use $\mathcal{L}_{\text{Distill}}$ in Eq.~(\ref{eq:distillloss}), which does not contain any additional operation on $P^s$. We follow the same model configurations as those in Section~\ref{section:maineresults}, while using the node splits of~\citep{yang2016ccp} with full-batch training in the transductive setting. Table~\ref{table:indanalysis} summarizes the experimental results for three different cases, $\mathcal{L}_\text{conv}$, $\mathcal{L}_\text{Distill}$, $\mathcal{L}_\text{InvKD}$, when three datasets including Cora, CiteSeer, and PubMed are used in the transductive setting. Interestingly, we can observe that the case of using $\mathcal{L}_{\text{conv}}$ exhibits much lower performance than that of other two cases. This is because, while using $\mathcal{L}_{\text{conv}}$ explicitly accommodates the graph structure during distillation, it rather alleviates the pressure for the student MLP to learn the graph structure due to the fact that $\tilde{A}$ is multiplied during training; however, during inference, the student MLP acquires low information on the graph structure with $\tilde{A}$ that is no longer injected, which eventually harms the performance. Hence, this implies that using a na\"ive alternative operation may deteriorate the performance and a judicious design of the propagation function is essential in guaranteeing state-of-the-art performance.

\subsection{Distillation from GNNs with Decoupled Propagation and Transformation Phases}

\begin{figure*}[t]
    \centering
    \subfigure[Cora.]{
    \centering
    \includegraphics[width=0.3\textwidth]{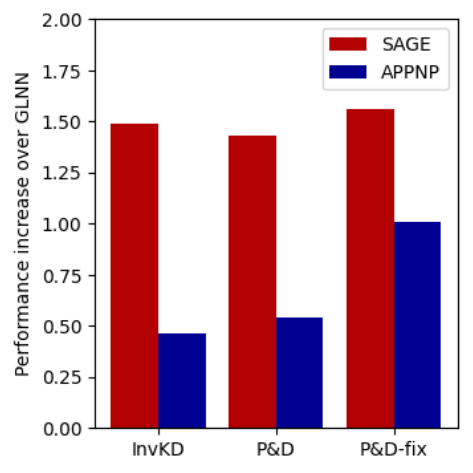}
    \label{figure2:main-a}
    }
    \centering
    \subfigure[Citeseer.]{
    \centering
    \includegraphics[width=0.3\textwidth]{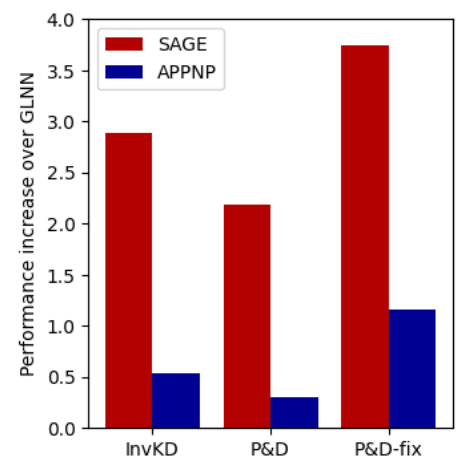}
    \label{figure2:main-b}
    }
    \centering
    \subfigure[Pubmed.]{
    \centering
    \includegraphics[width=0.3\textwidth]{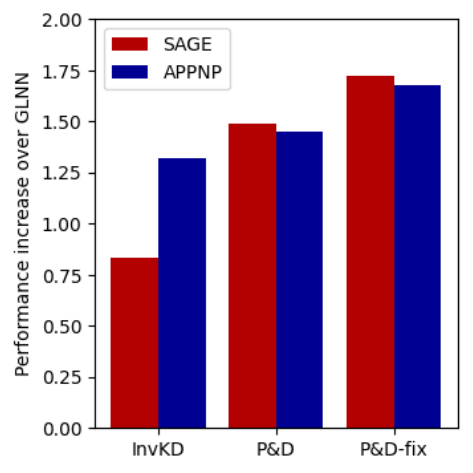}
    \label{figure2:main-c}
    }
    \caption{Performance increase of InvKD, \textbf{P\&D}, and \textbf{P\&D}-fix over GLNN across three different datasets under two different teacher GNNs. The red and blue bars in each subfigure indicate the performance increase over GLNN using GraphSAGE and APPNP as the teacher GNN, respectively.}
    \label{figure2:main}
\end{figure*}

In our main experiments, we have employed GraphSAGE~\citep{hamilton2017graphsage} as the teacher GNN. However, the performance of the student MLP may change under different types of GNN. Since our framework are inspired from such GNNs that decouple propagation and transformation~\citep{gasteiger2019ppnpappnp, huang2021cns, bojchevski2020scalegnnappnp, chien2021generalappnp}, we can expect that the performance benefits may change when we change the teacher GNN. To analyze this, we perform an additional experiment by replacing the teacher GNN with APPNP~\citep{gasteiger2019ppnpappnp}, a representative model that decouples propagation and transformation. 

Figure~\ref{figure2:main} shows the performance increase of InvKD, \textbf{P\&D}, and \textbf{P\&D}-fix over GLNN for the Cora, Citeseer, and Pubmed datasets under two different settings: employing GraphSAGE (red bar) and APPNP (blue bar) as the teacher GNNs, respectively. We observe that, except for one case, the benefit of our framework is diminished to some extent when the student MLP is distilled from APPNP. This is because the potential performance benefit during distillation becomes lower due to the similarities of the propagation functions used for both APPNP and our framework. Still, the student MLP is able to outperform GLNN for all cases, which is due to the capability of finding a better model parameter via the additional propagation-based regularization effect during distillation.

\section{Theoretical Analysis}
\label{section:theoreticalanalysis}
In this section, we provide a theoretical analysis of propagating the teacher's output $P^t$ along the graph in the context of {\it self-correction}, which was hinted from our earlier case study in Section~\ref{subsection:casestudy}. {\bf P\&D} relies on the assumption that the underlying graph has high levels of homophily, which measures the ratio of edges where the two connected nodes have the same class label~\citep{zhu2020h2gcn}. In this setting, we are interested in analyzing the condition where the prediction of a (incorrectly-predicted) node becomes corrected after one iteration of propagation by Eq.~(\ref{equation:labelpropagation}). We start by formally addressing basic settings and assumptions, which essentially follow those of~\citep{zhu2020h2gcn}. Let us assume that the underlying graph $G$ is regular ({\it i.e.}, all nodes have a degree of $d$) and $h\in[0,1]$ portion of neighbors have the same label for all nodes in $v\in\mathcal{V}$. For each node, the teacher GNN is assumed to be assigned an output probability vector having a probability $p \in [0,1]$ (with $p > 1/|\mathcal{Y}|$) for the true class label and another probability $(1-p)/(|\mathcal{Y}|-1)$ for the rest of the classes if the teacher GNN always makes predictions correctly.

Now, without loss of generality, let us assume that the teacher GNN makes an {\em incorrect} prediction for a particular node of interest $v_{*}$ with class 0 as its ground truth label by assigning a new output probability vector
\begin{equation}
    P^t[*,:] = \left[q, \dfrac{1-q}{|\mathcal{Y}|-1}, \cdots, \dfrac{1-q}{|\mathcal{Y}|-1}\right]
\end{equation}
with $0 < q < 1/|\mathcal{Y}|$, thus no longer assigning class 0 as its prediction. Additionally, for the \textit{rest} of the nodes, we introduce an error ratio $\epsilon \in (0, 1)$ to the teacher GNN's predictions, which assumes that the teacher model provides incorrect predictions for $\epsilon(|\mathcal{V}|-1)$ nodes (excluding $v_{*}$ itself). For the sake of simplicity, we assume that the probability of a node being incorrect by the teacher GNN is independent of its ground truth class label, which establishes the following theorem:

\begin{theorem}
\label{theorem:smoothingcorrectionwitherror}
    Suppose that the teacher GNN provides incorrect predictions for $\epsilon(|\mathcal{V}| - 1)$ nodes other than node $v_*$ where $\epsilon\in(0,1)$. Then, using one iteration of propagation in Eq.~(\ref{equation:labelpropagation}), the prediction of node $v_*$ gets corrected if
    \begin{equation}
    \label{equation:smoothingcorrectionwitherror}
        q \in \left[\max\left(0, \dfrac{1}{|\mathcal{Y}|} - \dfrac{\gamma}{1 - \gamma}\left( C - b(\epsilon)\right)\right), \dfrac{1}{|\mathcal{Y}|}\right],
    \end{equation}
    where $q$ is the output probability of $v_*$ corresponding to class 0, $C$ is approximately $\left(1 + \dfrac{1}{|\mathcal{Y}|}\right)hp - \dfrac{h+p}{|\mathcal{Y}|}$, $\gamma$ is the propagation strength in Eq.~(\ref{equation:labelpropagation}), and $b(\epsilon) = \left(C + \dfrac{hp}{|\mathcal{Y}|}\right)\epsilon$.
\end{theorem}
\begin{proof} 
We refer to~\ref{section:SupplHomophily} for the proof.
\end{proof}

From Theorem~\ref{theorem:smoothingcorrectionwitherror}, we can see that an increase of the error $\epsilon$ reduces the range of $q$, enabling the prediction of nodes to get corrected, which means that incorrect predictions from the teacher GNN introduce a more unforgiving environment for self-correction. Moreover, it is worth noting that the acceptable amount of error such that corrections via propagation are possible is upper-bounded by
\begin{equation}
    \epsilon < \frac{|\mathcal{Y}|h - 1}{(|\mathcal{Y}| + 1)h - 1},
\end{equation}
which monotonically increases with $h \in (1/|\mathcal{Y}|, 1]$. This implies that stronger homophily of the underlying graph will result in more tolerance of the prediction error from the teacher GNN.

\section{Conclusion}
\label{section:conclusion}
We presented {\bf P\&D} and its variant {\bf P\&D}-fix, simple yet effective GNN-to-MLP KD frameworks to boost the performance of MLP models trained by distillation from a teacher GNN model. We empirically showed that applying an approximate propagation $\Bar{\Pi}$ to the teacher GNN's output eventually benefits the student MLP model after KD on real-world graph benchmark datasets for both transductive and inductive settings. Our future work includes replacing the proposed propagation with a more sophisticated one to further improve the performance.

\appendix

\section{Proof of Theorem 2}
\label{section:SupplHomophily}

Before proving Theorem~\ref{theorem:smoothingcorrectionwitherror}, we first show preliminary calculations that are needed before we proceed with the main proof. Here, we start with the simplest case. Following~\citep{zhu2020h2gcn}, we calculate the result when one-hot labels are being propagated, which will be used for later calculations. First, without loss of generality, we reorder the label matrix $Y \in \{0,1\}^{|V| \times |\mathcal{Y}|}$ (and also the rows and columns of adjacency matrix $A$) as follows:
\begin{equation}
\Scale[0.8]{
    Y =
    \begin{bmatrix}
        1 & 0 & \cdots & 0 \\
        \vdots & \vdots & \ddots & \vdots \\
        1 & 0 & \cdots & 0 \\
        0 & 1 & \cdots & 0 \\
        \vdots & \vdots & \ddots & \vdots \\
        0 & 1 & \cdots & 0 \\
        \vdots & \vdots &  & \vdots \\
        0 & 0 & \cdots & 1 \\
        \vdots & \vdots & \ddots & \vdots \\
        0 & 0 & \cdots & 1 \\
    \end{bmatrix}.
}
\end{equation}
We then calculate $(A + I)Y$, which we will progressively modify to $\Bar{\Pi}$ during the rest of this section. We take advantage of the neighbor assumptions, which results in:
\begin{equation}
\label{eq:naivelp}
\Scale[0.8]{
    (A + I)Y =
    \begin{bmatrix}
        hd + 1 &  \dfrac{1-h}{|\mathcal{Y}| - 1}d & \cdots & \dfrac{1-h}{|\mathcal{Y}| - 1}d \\
        \vdots & \vdots & \ddots & \vdots \\
        hd + 1 & \dfrac{1-h}{|\mathcal{Y}| - 1}d & \cdots & \dfrac{1-h}{|\mathcal{Y}| - 1}d \\
        \dfrac{1-h}{|\mathcal{Y}| - 1}d & hd + 1 & \cdots & \dfrac{1-h}{|\mathcal{Y}| - 1}d \\
        \vdots & \vdots & \ddots & \vdots \\
        \dfrac{1-h}{|\mathcal{Y}| - 1}d & hd + 1 & \cdots & \dfrac{1-h}{|\mathcal{Y}| - 1}d \\
        \vdots & \vdots &  & \vdots \\
        \dfrac{1-h}{|\mathcal{Y}| - 1}d & \dfrac{1-h}{|\mathcal{Y}| - 1}d & \cdots & hd + 1 \\
        \vdots & \vdots & \ddots & \vdots \\
        \dfrac{1-h}{|\mathcal{Y}| - 1}d & \dfrac{1-h}{|\mathcal{Y}| - 1}d & \cdots & hd + 1 \\
    \end{bmatrix}.
}
\end{equation}

We now start to modify this result of calculating $(A + I)Y$ that more resembles Eq.~(\ref{equation:labelpropagation}), except that we are still propagating one-hot labels. First, we introduce $\gamma$ in Eq.~(\ref{eq:naivelp}) and calculate $(\gamma A + (1 - \gamma)I)Y$: 
\begin{equation}
\Scale[0.8]{
    (\gamma A + (1 - \gamma)I)Y =
    \begin{bmatrix}
        \gamma hd + (1 - \gamma) &  \gamma\dfrac{1-h}{|\mathcal{Y}| - 1}d & \cdots & \gamma\dfrac{1-h}{|\mathcal{Y}| - 1}d \\
        \vdots & \vdots & \ddots & \vdots \\
        \gamma hd + (1 - \gamma) & \gamma\dfrac{1-h}{|\mathcal{Y}| - 1}d & \cdots & \gamma\dfrac{1-h}{|\mathcal{Y}| - 1}d \\
        \gamma\dfrac{1-h}{|\mathcal{Y}| - 1}d & \gamma hd + (1 - \gamma) & \cdots & \gamma\dfrac{1-h}{|\mathcal{Y}| - 1}d \\
        \vdots & \vdots & \ddots & \vdots \\
        \gamma\dfrac{1-h}{|\mathcal{Y}| - 1}d & \gamma hd + (1 - \gamma) & \cdots & \gamma \dfrac{1-h}{|\mathcal{Y}| - 1}d \\
        \vdots & \vdots &  & \vdots \\
        \gamma\dfrac{1-h}{|\mathcal{Y}| - 1}d & \gamma\dfrac{1-h}{|\mathcal{Y}| - 1}d & \cdots & \gamma hd + (1 - \gamma) \\
        \vdots & \vdots & \ddots & \vdots \\
        \gamma\dfrac{1-h}{|\mathcal{Y}| - 1}d & \gamma\dfrac{1-h}{|\mathcal{Y}| - 1}d & \cdots & \gamma hd + (1 - \gamma) \\
    \end{bmatrix}.
}
\end{equation}

Finally, we replace $A$ with $\Tilde{A} = D^{-1/2} A D^{-1/2}$. Since each node has the same degree $d$, each signal is now multiplied with $(1/\sqrt{d})^2 = 1/d$ during propagation, eventually cancelling out the $d$'s:
\begin{equation}
\Scale[0.8]{
    (\gamma \Tilde{A} + (1 - \gamma)I)Y =
    \begin{bmatrix}
        \gamma h + (1 - \gamma) &  \gamma\dfrac{1-h}{|\mathcal{Y}| - 1} & \cdots & \gamma\dfrac{1-h}{|\mathcal{Y}| - 1} \\
        \vdots & \vdots & \ddots & \vdots \\
        \gamma h + (1 - \gamma) & \gamma\dfrac{1-h}{|\mathcal{Y}| - 1} & \cdots & \gamma\dfrac{1-h}{|\mathcal{Y}| - 1} \\
        \gamma\dfrac{1-h}{|\mathcal{Y}| - 1} & \gamma h + (1 - \gamma) & \cdots & \gamma\dfrac{1-h}{|\mathcal{Y}| - 1} \\
        \vdots & \vdots & \ddots & \vdots \\
        \gamma\dfrac{1-h}{|\mathcal{Y}| - 1} & \gamma h + (1 - \gamma) & \cdots & \gamma \dfrac{1-h}{|\mathcal{Y}| - 1} \\
        \vdots & \vdots &  & \vdots \\
        \gamma\dfrac{1-h}{|\mathcal{Y}| - 1} & \gamma\dfrac{1-h}{|\mathcal{Y}| - 1} & \cdots & \gamma h + (1 - \gamma) \\
        \vdots & \vdots & \ddots & \vdots \\
        \gamma\dfrac{1-h}{|\mathcal{Y}| - 1} & \gamma\dfrac{1-h}{|\mathcal{Y}| - 1} & \cdots & \gamma h + (1 - \gamma) \\
    \end{bmatrix}.
}
\end{equation}

Now, we are ready to change $Y$ to a matrix of probability vectors ({\it i.e.}, $P^t$), which finally results in calculating one iteration of $\Bar{\Pi}$. In our analysis, we replace $Y$ with a $P^t \in [0,1]^{|V| \times |\mathcal{Y}|}$, where the correct label is predicted with probability $p$ and the rest of the probabilities $(1-p)$ is distributed uniformly for the rest of the classes:
\begin{equation}
\Scale[0.8]{
    P^t =
    \begin{bmatrix}
        p & \dfrac{1-p}{|\mathcal{Y}| - 1} & \cdots & \dfrac{1-p}{|\mathcal{Y}| - 1} \\
        \vdots & \vdots & \ddots & \vdots \\
        p & \dfrac{1-p}{|\mathcal{Y}| - 1} & \cdots & \dfrac{1-p}{|\mathcal{Y}| - 1} \\
        \dfrac{1-p}{|\mathcal{Y}| - 1} & p & \cdots & \dfrac{1-p}{|\mathcal{Y}| - 1} \\
        \vdots & \vdots & \ddots & \vdots \\
        \dfrac{1-p}{|\mathcal{Y}| - 1} & p & \cdots & \dfrac{1-p}{|\mathcal{Y}| - 1} \\
        \vdots & \vdots &  & \vdots \\
        \dfrac{1-p}{|\mathcal{Y}| - 1} & \dfrac{1-p}{|\mathcal{Y}| - 1} & \cdots & p \\
        \vdots & \vdots & \ddots & \vdots \\
        \dfrac{1-p}{|\mathcal{Y}| - 1} & \dfrac{1-p}{|\mathcal{Y}| - 1} & \cdots & p \\
    \end{bmatrix}.
}
\end{equation}
Assuming $1/|\mathcal{Y}| < p \leq 1$ implies that the teacher GNN has the accuracy of 1 ({\it i.e.}, perfect prediction). Note that setting $p = 1$ reverts $P^t$ to $Y$. Now, calculating for $(\gamma \Tilde{A} + (1 - \gamma)I)P^t$ results in:
\begin{equation*}
\Scale[0.8]{(\gamma \Tilde{A} + (1 - \gamma)I)P^t=
    \begin{bmatrix}
        \beta &  \beta' & \cdots & \beta' \\
        \vdots & \vdots & \ddots & \vdots \\
        \beta & \beta' & \cdots & \beta' \\
        \beta' & \beta & \cdots & \beta' \\
        \vdots & \vdots & \ddots & \vdots \\
        \beta' & \beta & \cdots & \beta' \\
        \vdots & \vdots &  & \vdots \\
        \beta' & \beta' & \cdots & \beta \\
        \vdots & \vdots & \ddots & \vdots \\
        \beta' & \beta' & \cdots & \beta \\
    \end{bmatrix}
    },
\end{equation*}
where
\begin{align}
    \beta &= (1 - \gamma)p + \gamma hp + \gamma\dfrac{1-h}{|\mathcal{Y}| - 1}(1-p) \label{eq:beta_q}\\
    \beta' &= (1 - \gamma)\dfrac{1-p}{|\mathcal{Y}| - 1} + \gamma\dfrac{h}{|\mathcal{Y}| - 1}(1-p) + \gamma\dfrac{1-h}{|\mathcal{Y}| - 1}p + \gamma(1-h)\dfrac{|\mathcal{Y}| - 2}{(|\mathcal{Y}| - 1)^2}(1-p)\label{eq:beta'_q}.
\end{align}


Now, we are ready to prove Theorem~\ref{theorem:smoothingcorrectionwitherror}.

In order to prove the theorem, we basically follow the same calculation steps by recalculating the interval for $v_* = v_1$ while including $\epsilon$, which will require more careful considerations. We start with revisiting Eqs.~(\ref{eq:beta_q}) and (\ref{eq:beta'_q}). In Eq.~(\ref{eq:beta_q}), there are three terms, {\it i.e.}, $\gamma hp$, $(1 - \gamma)q$, and $\gamma(1-p)\dfrac{1-h}{|\mathcal{Y}|-1}$. For ease of notations, we denote the set of nodes with the same labels with $v_1$ as $S$ and the rest of neighbors as $S'$ in the previous setting where we assumed that nodes other than $v_1$ were all correct.

Starting with $(1 - \gamma)q$, this term calculates the effect of self-propagation and therefore remains unchanged in the new setting. The term $\gamma hp$ calculates the influence that is aggregated from nodes in $S$. In the new setting, only $(1-\epsilon)hd$ nodes propagate the probability $p$ (Note that $|S| = hd$), and therefore $\gamma hp$ is changed into $\gamma h\left((1-\epsilon)p + \epsilon \dfrac{1 - p}{|\mathcal{Y}| - 1}\right)$. Next, the term $\gamma(1-h)\dfrac{1-p}{|\mathcal{Y}|-1}$ calculates the influence aggregated from nodes in $S'$, which previously all propagated $\dfrac{1-p}{|\mathcal{Y}|-1}$. When $\epsilon=0$, the number of these nodes is $|S'|=(1-h)d$, and in the new setting, $\epsilon (1-h)d$ has their predictions changed, where $\dfrac{1}{|\mathcal{Y}|-1}$ of them now propagates $p$. Therefore, this term is now changed into $\gamma (1-h)\left( \epsilon \dfrac{1}{|\mathcal{Y}|-1} p + \dfrac{|\mathcal{Y}|-1 - \epsilon}{(|\mathcal{Y}|-1)^2}(1-p)\right)$. In summary, in the new setting, $\beta_q$ becomes 

\begin{align}
    \beta_{q, \epsilon} &= (1-\gamma)q + \gamma h\left((1-\epsilon)p + \epsilon \dfrac{1 - p}{|\mathcal{Y}| - 1}\right) \notag \\
    &+ \gamma (1-h)\left( \epsilon \dfrac{1}{|\mathcal{Y}|-1} p + \dfrac{|\mathcal{Y}|-1 - \epsilon}{(|\mathcal{Y}|-1)^2}(1-p)\right).
\end{align}

In Eq.~(\ref{eq:beta'_q}), there are four terms, {\it i.e.}, $(1 - \gamma)\dfrac{1-q}{|\mathcal{Y}| - 1}$, $\gamma\dfrac{h}{|\mathcal{Y}| - 1}(1-p)$, $\gamma\dfrac{1-h}{|\mathcal{Y}| - 1}p$, and $\gamma(1-h)\dfrac{|\mathcal{Y}| - 2}{(|\mathcal{Y}| - 1)^2}(1-p)$. Using the assumption that the predictions are uniformly distributed among classes for nodes in $S'$; without loss of generality, let us calculate the probability regarding the second class label.

Similarly as in $\beta_q$, the term $(1 - \gamma)\dfrac{1-q}{|\mathcal{Y}| - 1}$ remains unaffected in the new setting as it is the result of self-propagation. The term $\gamma\dfrac{h}{|\mathcal{Y}| - 1}(1-p)$ calculates the influence from nodes that were in $S$. For these $|S| = hd$ nodes, they previously propagated $\dfrac{1-p}{|\mathcal{Y}| - 1}$. In the new setting, $\epsilon \dfrac{1}{|\mathcal{Y}| - 1} hd $ nodes now predict the second class and propagate $p$, while the rest of the $\dfrac{|\mathcal{Y}| - 1 - \epsilon}{|\mathcal{Y}| - 1}hd$ nodes still propagates $\dfrac{1-p}{|\mathcal{Y}| - 1}$. In total, this term is now modified into $\gamma h \left(\dfrac{\epsilon}{|\mathcal{Y}| - 1}p + \dfrac{|\mathcal{Y}| - 1 - \epsilon}{(|\mathcal{Y}| - 1)^2}(1-p)\right)$. The term $\gamma\dfrac{1-h}{|\mathcal{Y}| - 1}p$ calculates the influence from nodes that previously predicted the second class in $S'$. The nodes are now split into two groups with ratio $\epsilon : (1 - \epsilon)$, where the first group now propagates $\dfrac{1-p}{|\mathcal{Y}| - 1}$ and the latter still propagates $p$. In total, this term is now modified into $\gamma \dfrac{1-h}{|\mathcal{Y}| - 1}\left(\epsilon \dfrac{1}{|\mathcal{Y}| - 1}(1-p) + (1-\epsilon)p\right)$. Next, the term $\gamma(1-h)(1-p)\dfrac{|\mathcal{Y}| - 2}{(|\mathcal{Y}| - 1)^2}$ calculates the influence from nodes that previously did not predict as the second class in $S'$. Similarly as before, the nodes are now split into two groups with ratio $\dfrac{1}{|\mathcal{Y}| - 1}\epsilon : \dfrac{|\mathcal{Y}| - 1 - \epsilon}{|\mathcal{Y}| - 1}$, where the first group now (incorrectly) predicts the second class and thus propagates $p$, while the latter still propagates $\dfrac{1-p}{|\mathcal{Y}| - 1}$. In total, this term is now modified into $\gamma(1-h)\dfrac{|\mathcal{Y}| - 2}{|\mathcal{Y}| - 1}\left(\dfrac{\epsilon}{|\mathcal{Y}| - 1}p + \dfrac{|\mathcal{Y}| - 1 - \epsilon}{(|\mathcal{Y}| - 1)^2}(1-p)\right)$. In summary, in the new setting, $\beta'_q$ becomes 
\begin{align}
    \beta'_{q, \epsilon} &= (1 - \gamma)\dfrac{1-q}{|\mathcal{Y}| - 1} + \gamma h \left(\dfrac{\epsilon}{|\mathcal{Y}| - 1}p + \dfrac{|\mathcal{Y}| - 1 - \epsilon}{(|\mathcal{Y}| - 1)^2}(1-p)\right) \nonumber \\ 
    &+ \gamma \dfrac{1-h}{|\mathcal{Y}| - 1}\left(\epsilon \dfrac{1}{|\mathcal{Y}| - 1}(1-p) + (1-\epsilon)p\right) \nonumber\\ 
    &+ \gamma(1-h)\dfrac{|\mathcal{Y}| - 2}{|\mathcal{Y}| - 1}\left(\dfrac{\epsilon}{|\mathcal{Y}| - 1}p + \dfrac{|\mathcal{Y}| - 1 - \epsilon}{(|\mathcal{Y}| - 1)^2}(1-p)\right).
\end{align}

We can verify that both $\beta_{q, \epsilon = 0}$ and $\beta'_{q, \epsilon = 0}$ reduce to $\beta_{q}$ and $\beta'_q$, respectively. Now, in the scenario where the node prediction is corrected after propagation, we need $\beta_{q, \epsilon} > \beta'_{q, \epsilon}$. After calculation with similar approximations when we calculated Eq.~(\ref{equation:qwithgoodteacher}), we arrive at:
\begin{align}
    q &> \dfrac{1}{|\mathcal{Y}|} - \dfrac{\gamma}{1 - \gamma}\left(\left((1 - \epsilon) + \dfrac{1 - 2\epsilon}{|\mathcal{Y}|}\right)hp - (1-\epsilon)\dfrac{h + p}{|\mathcal{Y}|}\right) \notag\\ 
    &= \dfrac{1}{|\mathcal{Y}|} - \dfrac{\gamma}{1 - \gamma}\left(C - \epsilon\left(C + \dfrac{hp}{|\mathcal{Y}|}\right)\right),
\end{align}
which concludes the proof of Theorem~\ref{theorem:smoothingcorrectionwitherror}.


\section{Analysis for $\epsilon=0$}
Let us also consider a simpler scenario where the teacher GNN makes an incorrect prediction \textit{only} for a single node $v_*$ ({\it i.e.}, $\epsilon = 0$) with class 0 as its ground truth label by assigning a new output probability vector $P^t[*,:] = [q, \dfrac{1-q}{|\mathcal{Y} - 1|}, \cdots, \dfrac{1-q}{|\mathcal{Y} - 1|}]$ ($0 \leq q < 1/|\mathcal{Y}|$), same as in our previous setting in Theorem~\ref{theorem:smoothingcorrectionwitherror}. In this setting, we establish the following corollary.

\begin{corollary}
\label{corollary:appendix}
Suppose that the teacher make an incorrect prediction only for a single node $v_*$. Using one iteration of propagation in Eq.~(\ref{equation:labelpropagation}), the prediction of node $v_*$ gets corrected if
    \begin{equation}
        q \in \left[\max \left(0, \dfrac{1}{|\mathcal{Y}|} - \dfrac{\gamma}{1 - \gamma} C \right), \dfrac{1}{|\mathcal{Y}|}\right],
    \end{equation}
where $q$ is the output probability of $v_*$ corresponding to class 0 and $C$ is approximately $\left(1 + \dfrac{1}{|\mathcal{Y}|}\right)hp - \dfrac{h+p}{|\mathcal{Y}|}$.
\end{corollary}

\begin{proof}
Since we consider the case where $\epsilon=0$ ({\it i.e.}, the teacher GNN provides correct predictions for nodes other than $v^{*}$), the resulting vector for the first row of $(\gamma \Tilde{A} + (1 - \gamma)I)P^t$ can be expressed as:
\begin{equation}
    ((\gamma \Tilde{A} + (1 - \gamma)I)P^t)[1,:] = [\beta_q, \underbrace{\beta_q', \cdots, \beta_q'}_{(|\mathcal{Y}| - 1)}].
\end{equation}

Intuitively, $\beta_q$ represents the result after propagation for the correct class, and $\beta'_q$ represents the rest of the (incorrect) classes. For the incorrect prediction to be corrected after propagation, it requires $\beta_q >\beta'_q$:
\begin{align}
    (1 - \gamma)\left(q - \dfrac{1-q}{|\mathcal{Y}|-1}\right) &> -\gamma hp - \gamma \dfrac{1-h}{|\mathcal{Y}|-1}(1-p) + \gamma\dfrac{h}{|\mathcal{Y}| - 1}(1-p) \nonumber \\
    &+ \gamma\dfrac{1-h}{|\mathcal{Y}| - 1}p + \gamma(1-h)\dfrac{|\mathcal{Y}| - 2}{(|\mathcal{Y}| - 1)^2}(1-p). \label{eq:intermediate}
\end{align}
Calculation of Eq.~(\ref{eq:intermediate}) can directly reveal the condition for the correction scenario. We can directly derive an interval for $q$ by approximating $\dfrac{|\mathcal{Y}|-2}{|\mathcal{Y}|-1} \approx 1$, which further reduces Eq.~(\ref{eq:intermediate}) to
\begin{equation}
\label{equation:qwithgoodteacher}
    q > \dfrac{1}{|\mathcal{Y}|} - \dfrac{\gamma}{1 - \gamma}\left(\left(1 + \dfrac{1}{|\mathcal{Y}|}\right)hp - \dfrac{h + p}{|\mathcal{Y}|}\right).
\end{equation}

Denoting $C = \left(1 + \dfrac{1}{|\mathcal{Y}|}\right)hp - \dfrac{h+p}{|\mathcal{Y}|}$ results in the interval Eq.~(\ref{eq:intermediate}), which concludes the proof of Corollary~\ref{corollary:appendix}.

\end{proof}

\section*{Acknowledgments}
This work was supported by the National Research Foundation of Korea (NRF) grant funded by the Korea government (MSIT) under Grants 2021R1A2C3004345 and RS-2023-00220762 and by the Institute of Information and Communications Technology Planning and Evaluation (IITP), Republic of Korea Grant by the Korean Government through MSIT (6G Post-MAC--POsitioning and Spectrum-Aware intelligenT MAC for Computing and Communication Convergence) under Grant 2021-0-00347.


\bibliographystyle{elsarticle-num-names} 
\bibliography{cas-refs}





\end{document}